\documentclass[10pt,a4paper]{article}
\usepackage{preamble}

\title{Distribution estimation via Flow Matching with Lipschitz guarantees}
\author{Lea Kunkel\footnote{The author thanks Mathias Trabs for inspiring discussions and many helpful comments on this project.}}
\date{Karlsruhe Insitute of Technology\thanks{The author is now at Ruhr University Bochum.}}

\begin{document}
	\maketitle
	\begin{center}
	\begin{minipage}{0.8\textwidth}
         Flow Matching, a promising approach in generative modeling, has recently gained popularity. Relying on ordinary differential equations, it offers a simple and flexible alternative to diffusion models, which are currently the state-of-the-art. Despite its empirical success, the mathematical understanding of its statistical power so far is very limited. This is largely due to the sensitivity of theoretical bounds to the Lipschitz constant of the vector field which drives the ODE. In this work, we study the assumptions that lead to controlling this dependency. Based on these results, we derive a convergence rate for the Wasserstein $1$ distance between the estimated distribution and the target distribution which improves previous results in high dimensional setting. This rate applies to certain classes of unbounded distributions and particularly does not require $\log$-concavity.

	\end{minipage}
\end{center}

	\section{Introduction}
Generative models aim to learn a mapping $\psi$ that pushes a simple latent variable $Z \sim \mathbb{U}$ towards a target distribution $\mathbb{P}^*$ on $\mathbb{R}^d$. This enables quick generation through sampling from the latent distribution and applying the learned mapping.

Flow Matching, introduced by \cite{lipman2023}, has recently attracted attention for its simplicity compared to diffusion models \citep{sohl2015}, which have long been regarded as state-of-the-art. Unlike generative adversarial networks  \citep{goodfellow}, which learn $\psi$ directly, Flow Matching learns a time-depended vector field $v_t$ that transports mass from $\mathbb{U}$ to an approximation of $\mathbb{P}^*$ via the ODE for $t \in [0,1]$:
\begin{equation}\label{ODE_intro}
    \frac{d}{d t} \psi_t(x)=v_t\left(\psi_t(x)\right), \quad \psi_0(x)=x
\end{equation}
inducing path densities $p_t = [\psi_t]_{\#}p_0$, where $p_0$ is the density of the latent distribution $\mathbb{U}$. The vector field $v$ is chosen such that $p_1$ approximates the target density.

Earlier, normalizing flows \citep{tabak2010, tabak2013}, popularized by \citet{rezende2015} and \citet{dinh2015} used discrete invertible mappings with the change-of-variables formula. Continuous normalizing flows \citep{chen2018} extended this to continuous mappings using neural ODEs trained via maximum likelihood, which circumvented the calculation of the determinant of the Jacobian. However, CNFs still require significant computational costs due to the necessity of ODE simulations \citep{grathwohl2018}.
Flow Matching avoids these simulations by using a simple least-squares loss
\begin{equation}\label{eq:FMobjective_intro}
\begin{gathered}
\Psi(\tilde{v})\coloneqq\mathbb{E}_{\substack{t \sim \mathcal{U}[0,1],\\ Y \sim p^*, \\X_t \sim p_t(\cdot \mid Y)}}\left[\left|\tilde{v}_t\left(X_t\right)-v_t\left(X_t \mid Y\right)\right|^2\right]
\end{gathered}
\end{equation}
where $v_t(\cdot|Y)$ and $p_t(\cdot|Y)$ are known by design. In some settings, this resembles the score matching objective \citep{hyvarinen2005, vincent2011} used in diffusion models \citep{song2020}. Given $n$ i.i.d. observations from $\mathbb{P}^*$, minimizing the empirical counterpart of \eqref{eq:FMobjective_intro} over a set of functions $ \mathcal{M}$ and solving the ODE \eqref{ODE_intro} leads to the estimator $\hat{\psi}$. The statistical question is how well $\mathbb{P}^{\hat{\psi}_1(Z)}$ approximates $\mathbb{P}^*$ depending on the number of observations $n$ and the function class $\mathcal{M}$.

Algorithms based on Flow Matching have been successfully applied in a range of domains, including text-to-speech \citep{guo2024}, text-to-image generation \citep{yang2024, esser2024}, molecular and protein structure design \citep{dunn2024, bose2024}, and surrogate modeling in high-energy physics \citep{bieringer2024classifier}. The method has also been extended to various theoretical settings: \cite{atanackovic2024} adapted it to interacting particle systems, \cite{gat2024} explored discrete spaces, \cite{chen2024} generalized it to Riemannian manifolds, and \cite{kerrigan2023} extended it to function spaces. 

The analysis of Flow Matching is more fragile compared to diffusions, as Girsanov's theorem does not apply, thus making the strategy used in \cite{chen2023c} inapplicable. Using Grönwall's Lemma, the natural and standard approach for stability estimates of ODEs, results in exponential dependency on the Lipschitz constant of the approximated vector field. Consequently, the statistical results so far are rather limited: \cite{fukumizu2024} use a slight adaptation and introduce stopping times to transfer some results from diffusion models but they do not circumvent the growth of the Lipschitz constant, which causes difficulties in their proof. \cite{gao2024} use the same adaptation and obtain rates in the Wasserstein 2 distance for certain types of unknown distributions.
\cite{stephanovitch2025generalization} study deterministic ODE sampling, but, due to their general setting, examine the infinite-time setting of diffusions.  \cite{kunkel2025minimax} showed that Flow Matching attains minmax optimal rates, but they use an overparameterized setting which does not align with practical implementations. Due to the excess of parameters, they are able to compensate the rising Lipschitz constant.

In this work, we first focus on a detailed study of the Lipschitz constant of the intrinsic "true" vector field of Flow Matching. By providing upper and lower bounds, we demonstrate that the choice of the variance function and the behavior of the covariance of the reweighted unknown distribution is crucial. This function governs the level of smoothing depending on $t \in [0,1]$. While current results focus on fixed choices of variance functions, we study the Lipschitz constant in terms of this function. We derive assumptions on the unknown distribution that lead to control over the Lipschitz constant for various types of variance functions. Our general assumptions are validated for a class of distributions with unbounded support considered by \cite{stephanovitch2025generalization}.

Subsequently, we employ Bernstein-type bounds to derive an oracle inequality for estimating the vector field. We carefully adapt these classical methods to address the theoretical challenges of Flow Matching. When bounding the error of the Flow Matching distribution estimator in the Wasserstein $1$ distance, standard approaches from approximation theory usually increase the negative impact of the dimension on the rate. This issue can be circumvented up to a small error by exploiting the smoothness of the vector field. Using this approach allows us to build on existing results on network approximation and thus keep the focus on the distributional aspects of Flow Matching. We consider feedforward rectified linear unit (ReLU) networks, hence we do not require smooth networks for our results. ReLU nets with logarithmically increasing depth and a polynomial number of non-zero weights suffice to achieve improved convergence rates. Additionally, we exploit the density path of the true vector field to take advantage of the smoothness of the unknown distribution. Thus, our rates are faster in high dimensions than those previously obtained in \cite{gao2024}, while allowing for considerably smaller networks than in \cite{kunkel2025minimax}.

\paragraph{Further Related work}
      Approaches related to flows between two possibly unknown distributions, $\mathbb{P}$ and $\mathbb{Q}$, are studied by \cite{tong2024}, \cite{liu2022} and \cite{albergo2022}. \cite{tong2024} also generalized and unified these methods.
 \cite{benton2024} studied Flow Matching, excluding approximation errors, by imposing assumptions on the covariance that provided global Lipschitz bounds on the vector field. The work by \cite{gong2025} examines the properties of ReLU networks used to approximate a vector field corresponding to higher-order trajectories.
 
As \cite{lipman2023} noted, Flow Matching is closely connected to diffusion models. For a comprehensive survey of generative diffusion models, refer to \cite{cao2024}. Even when the model is not built to be consistent with diffusion models, approximating a score function is similar to approximating the vector field in Flow Matching.
The statistical analysis of score matching and diffusion models are an active area of research, see for instance \cite{chen2023a}, \cite{chen2023b}, \cite{chen2023c}, \cite{oko2023}, \cite{tang24}, \cite{azangulov2024}, \cite{zhang2024} , \cite{yakovlev2025}, \cite{stephanovitch2025generalization}.

    \paragraph{Outline} 
\Cref{sec:empirical_flow_matching} introduces the problem setting mathematically and recalls all the definitions needed to define the conditional Flow Matching objective from \cite{lipman2023}. The section then discusses the use of the Wasserstein distance and shows a classical first error decomposition using the aforementioned Grönwalls lemma. After that, we introduce the model based on the choice $\mathbb{U} =\mathcal{N}(0, I_d)$, which will be the focus of the subsequent work. In \Cref{sec:lipschitz_constant} we establish upper and lower bounds of the Lipschitz constant. We then show that bounding the Lipschitz constant requires careful investigation of the covariance of the reweighted unknown distribution. Next, we demonstrate that, under certain assumptions, we can construct Flow Matching such that the Lipschitz constant is bounded. In \Cref{sec:rate} we derive an oracle inequality and finally obtain a rate of convergence using feedforward ReLU networks. All proofs are deferred to \Cref{sec:proofs}.
\subsection{Some notations}
	
We use the notation $X \sim p$ as shorthand for $X \sim \mathbb{P}$ when the distribution $\mathbb{P}$ of the random variable $X$ admits a density $p$ with respect to the Lebesgue measure.
The symbol $\lesssim$ denotes an inequality of the form $a \leq c \cdot b$, where $a, b \in \mathbb{R}$ and $c$ is a constant independent of the sample size $n$. In the same setting, the symbol $a \sim b$ denotes equality up to a constant.
For $d \in \mathbb{N}$, we equip $\mathbb{R}^d$ with the standard Euclidean norm $|\cdot|$. For a matrix $A \in \mathbb{R}^{d \times d}$, we denote the spectral norm by $\|A\|$. Given a function $f\colon \Omega \rightarrow \mathbb{R}^{d^{\prime \prime}}$, where $\Omega \subset$ $\mathbb{R}^{d^{\prime}}$, the supremum norm and the $L_1$-norm are defined by 
$
        \| f\|_{\infty, \Omega} \coloneqq \sup_{x \in \Omega} |f(x)|\quad\text{and}\quad
\| f\|_{1, \Omega} \coloneqq  \int_{\Omega} |f(x)| \; \mathrm{d}x.
$
When $\Omega=\mathbb{R}^{d^{\prime}}$, we abbreviate $\|f\|_{\infty, \mathbb{R}^{d^{\prime}}}$ as $\|f\|_{\infty}$, and similarly $\|f\|_{1, \mathbb{R}^{d^{\prime}}}=\|f\|_1$.
For $d^{\prime \prime}=1$ and $\alpha \in(0,1]$, the Besov space $B_{1, \infty}^\alpha(\Omega)$ is defined as
$
	B_{1, \infty}^{\alpha}(\Omega) \coloneqq \big\{ f \in L_1(\Omega) \colon  |f|_{	B_{1, \infty}^{\alpha}(\Omega)}\coloneqq \sup_{t>0} t^{-\alpha} 	\omega_1(f, t)_1 < \infty\big\},
$
where the modulus of smoothness is given by
$
	\omega_1(f, t)_1\coloneqq \sup _{0<|h| \leq t}\int |f(x)-f(x+h)|\; \mathrm{d}x, \; t>0, x \in \Omega.
$ If $x+h \notin \Omega$, the integrand is taken to be zero. The norm on $B_{1, \infty}^\alpha(\Omega)$ is defined by
$
\|f\|_{B_{1, \infty}^\alpha(\Omega)}\coloneqq\|f\|_{L_1(\Omega)}+|f|_{B_{1, \infty}^\alpha(\Omega)} .
$
Lastly, $C^k(\Omega)$ denotes the set of functions whose coordinate components are $k$-times continuously differentiable with bounded derivatives (in the supremum norm). For a multi-index $k \in \mathbb{N}_0^d$ with $|k|=$ $\sum_{i=1}^d k_i$, we write
$
D^k=\frac{\partial^{|k|}}{\partial x_1^{k_1} \cdots \partial x_d^{k_d}},
$
where $\frac{\partial}{\partial x_i}$ denotes the weak partial derivative.
For a function $f\colon \Omega \rightarrow \mathbb{R}$ we denote for $\beta \in \mathbb{N}$
$
\|f\|_{C^{\beta}} \coloneqq  \max_{k \colon |k|\leq \beta}  \|D^k f \|_{\infty}.
$
For the Jacobian of a vector valued function with respect to the vector $x$ we write $D_x$. The Gradient is denoted by $\nabla$ and the Hessian is denoted by $\nabla^2$. The covering number of a set $\mathcal{A}$ with respect to the norm $\|\cdot\|$ and the $\tau $-covering is denoted by $\mathcal{N}(\tau, \mathcal{A}, \| \cdot \|)$.

\section{The Flow Matching objective and Wasserstein evaluation}\label{sec:empirical_flow_matching}

Suppose we observe an i.i.d. sample $X_1^*, \ldots, X_n^*$ drawn from an unknown distribution $\mathbb{P}^*$.  Further assume that $\mathbb{P}^*$ admits a density $p^*$ with respect to the Lebesgue measure. The aim of generative modeling is to mimic the distribution $\mathbb{P}^*$.

To this end, consider a time-dependent vector field $v\colon[0,1] \times \mathbb{R}^d \rightarrow \mathbb{R}^d$, and define the flow $\psi\colon[0,1] \times$ $\mathbb{R}^d \rightarrow \mathbb{R}^d$ as the solution to the ODE:
\begin{equation}\label{ODE}
	\frac{d}{dt} \psi_t(x) = v_t(\psi_t(x)), \quad \psi_0(x) = x.
\end{equation}
Given a fixed latent distribution $\mathbb{U}$ with Lebesgue density $p_0$, the vector field $v$ induces a path of probability densities $p\colon[0,1] \times \mathbb{R}^d \rightarrow \mathbb{R}_{>0}$, defined via the push-forward:
\begin{equation*}
p_t=\left[\psi_t\right]_{\#} p_0, \quad \text { meaning } \quad \psi_t(Z) \sim p_t \text { for } Z \sim p_0,
\end{equation*}
with $\int p_t(x) \mathrm{d} x=1$ for all $t \in[0,1]$. \cite{lipman2023} constructed probability paths resulting from a convolution and a corresponding vector field,
	\begin{equation}\label{marginal_prob_paths}
		p_t(x)=\int p_t(x | y) p^*(y) \; \mathrm{d}y, \quad v_t(x)=\int v_t(x | y) \frac{p_t(x | y) p^*(y)}{p_t(x)} d y.
	\end{equation}
where $p_t(\cdot | y)\colon \mathbb{R}^d \rightarrow \mathbb{R}$  is a conditional probability path induced through some vector field $v_t(\cdot|y)\colon \mathbb{R}^d \rightarrow \mathbb{R}^d$ for $y \in \mathbb{R}^d$. It can be shown that $v$ induces $p$ in \eqref{marginal_prob_paths}. In order to approximate $v$ with a parameterized function $\tilde{v}$, which is a neural network in practice, \cite{lipman2023} consider the following objective
\begin{equation}\label{eq:fmo}
	\mathbb{E}_{\substack{t \sim \mathcal{U}[0,1]\\X_t \sim p_t}} \big[|  \tilde{v}_t(X_t) -v_t(X_t)|^2\big] = 	\mathbb{E}_{\substack{t\sim \mathcal{U}[0,1], \\Y \sim p^*, \\X_t \sim p_t\left(\cdot | Y\right)}}\big[\left|\tilde{v}_t(X_t)-v_t\left(X_t| Y\right)\right|^2\big] + C,
\end{equation}
where $C$ is a constant that is explicitly calculated in \Cref{thm:equivalence_constant}.
In practice, $p^*$ is unknown and therefore the empirical counterpart of \eqref{eq:fmo} in minimized over some class of parameterized function $\mathcal{M}$. \cite{kunkel2025minimax} show that there is also an equivalence of gradients in the empirical setting. The estimated flow $\hat{\psi}$ is obtained by
\begin{equation}\label{ODEhat}
	\frac{d}{dt} \hat \psi_t(x) = \hat v_t(\hat\psi_t(x)), \quad \hat\psi_0(x) = x,\qquad\text{for}\quad \hat v\in\operatorname*{argmin}_{\tilde v\in\mathcal M} \frac{1}{n} \sum_{i = 1}^n \mathbb{E}_{\substack{t\sim \mathcal{U}[0,1], \\X_t \sim p_t\left(\cdot | Y_i\right)}}\big[\left|\tilde{v}_t(X_t)-v_t\left(X_t| Y_i\right)\right|^2\big].
\end{equation}

We use this flow to push forward the known, latent distribution to time $t = 1$. In accordance with the goal of generative modeling, the distribution of this pushforward, $\mathbb{P}^{\psi_1(Z)}$, should mimic $\mathbb{P}^*.$\\

 For this purpose, it is necessary to measure the gap between $\mathbb P^*$ and $\mathbb{P}^{\hat{\psi}_1(Z)}$ with flow $\hat\psi$ from \eqref{ODEhat}. We adopt the Wasserstein-$1$ metric for evaluation.

On the normed vector space $(\mathbb{R}^d, |\cdot|)$ the Wasserstein-$1$ distance between two probability distributions $\mathbb{P}$ and $\mathbb{Q}$ on $\mathbb{R}^d$ is given by
\begin{equation} \label{wasserstein_1_def}
		\mathsf{W}_1(\mathbb{P}, \mathbb{Q})   \coloneqq \inf _{\pi \in \Pi(\mathbb{P}, \mathbb{Q})} \int_{\mathcal{X}} |x-y| \; \mathrm{d}  \pi(x, y) = \sup_{W \in \operatorname{Lip}(1)}\mathbb{E}_{\substack{X \sim \mathbb{P}\\ Y \sim \mathbb{Q}}}[W(X)-W(Z)],
\end{equation}
where $\Pi(\mathbb{P}, \mathbb{Q})$ denotes the set of couplings with marginals $\mathbb{P}$ and $\mathbb{Q}$ and the equality is shown in \citet[Theorem 5.10(i)]{Villani2008}. On the space of probability measures with finite first moment, the Wasserstein-$1$ distance characterizes weak convergence. A detailed discussion of its advantages over other metrics with the same characteristic can be found in \citet[p. 98 f.]{Villani2008}. Due to its geometric interpretability and robustness in high-dimensional, manifold-structured data settings, the Wasserstein-$1$ metric has commonly been used to evaluate generative models. Examples include: \cite{Schreuder2020}, \cite{Liang2018}, \cite{stephanovitch2023} or \cite{kunkel2025}.\\

The following theorem provides a first comparison of the performance of $\mathbb{P}^{\hat{\psi}_1(Z)}$ and $\mathbb{P}^{\psi_1(Z)}$, where $\psi$ is the solution of the ODE \eqref{ODE} using \eqref{marginal_prob_paths} as the vector field. In order to bound $\mathsf{W}_1(\mathbb{P}^{\hat{\psi}_1(Z)},\mathbb{P}^{\psi_1(Z)})$ in terms of the difference between the underlying vector fields, we use Grönwalls Lemma as the classical tool to control ODE stability. This leads to the typical exponential bound on the Lipschitz constant in space. The proof follows the along lines of Proposition 3 in \cite{albergo2022}.

\begin{theorem}\cite[Theorem 3.1]{kunkel2025minimax}\label{thm:error_decomp_easy}
     Assume that all functions in $\mathcal{M}$ are Lipschitz continuous for fixed $t$ with Lipschitz constant $\tilde{\Gamma}_t.$  Then for any $\tilde{v}\colon [0,1] \times\mathbb{R}^d \rightarrow \mathbb{R}^d,$ with $\tilde{v} \in \mathcal{M}$ and $v$ from \eqref{marginal_prob_paths}	
		\begin{equation}
			\mathsf{W}_1(\mathbb{P}^*, \mathbb{P}^{\hat{\psi}_1(Z)}) \leq \mathsf{W}_1(\mathbb{P}^*, \mathbb{P}^{\psi_1(Z)}) + \sqrt{2 e}
            e^{\int_{0}^1 \tilde{\Gamma}_t\; \mathrm{d}t} \Big(\mathbb{E}_{\substack{t \sim \mathcal{U}[0,1]\\X_t \sim p_t}} \big[|v_t(X_t) - \tilde{v}_t(X_t) |^2\big]\Big)^{1/2}
            . \label{basic_orakel}
		\end{equation}
\end{theorem}
 While this decomposition into the convolution error $\mathsf{W}_1(\mathbb{P}^*, \mathbb{P}^{\psi_1(Z)}) $ and the error depending on generalization and approximation is natural from a theoretical perspective, it cannot be reversed. As we shall see, different constructions of $v_t(\cdot|y)$ and $p_t(\cdot|y)$ can lead to the same distribution $\mathbb{P}^{\hat{\psi}_1(Z)}$. In case the other construction admits a vector field $\tilde{v} \in \mathcal{M}$, the right hand side of \Cref{thm:error_decomp_easy} would be large, while the left hand side is small. We note that the Lipschitz assumption can be relaxed to a one-sided Lipschitz assumption, following the proof of \citet[Proposition 1]{stephanovitch2025regularity}. While this leads to a bound on the largest eigenvalue of $D_x v$ instead of the largest entry, we refrain from this generalization to be consistent with previous results and assumptions in the analysis of Flow Matching \citep{albergo2022, benton2024} and to make the following chapter more intuitive.

The second term in \eqref{basic_orakel} is connected to the suitability of the set $\mathcal{M}$ to approximate $v$. This depends on the specific structure of $v$. The factor $ \exp(\int_{0}^1 2\tilde{\Gamma}_t\; \mathrm{d}t) $ results from the use of Grönwall's inequality and is one of the key difficulties in the analysis of Flow Matching. In the analysis of diffusion models, this can be circumvented using Girsanovs lemma to obtain bounds in the total variation distance. Intuitively, the Brownian motion prevents the exponential propagation of the error. The next example demonstrates this difference when bounding divergences between distributions with underlying dynamics. It also shows that the exponential factor cannot be avoided in a general setting.
	\begin{example}\label{ex:sharp} Let $d = 1,\lambda > 0$ and $Z \sim \mathcal{N}(0,1)$.
		Consider a linear vector field $v_t(x) = \lambda x $ and a perturbed version $v_t^{\varepsilon}(x) = \lambda x + \varepsilon$.
		Then $\psi_t(x) = e^{\lambda t}x$ is the solution of the ODE \eqref{ODE} using $v_t$ and $\psi^{\varepsilon}_t (x) = e^{\lambda t}x + \frac{\varepsilon}{\lambda}(e^{\lambda t}-1)$ is the solution of the ODE \eqref{ODE} using $v_t^{\varepsilon}$. Thus
		\begin{equation*}
			\mathsf{W}_1(\mathbb{P}^{\psi_1(Z)}, \mathbb{P}^{\psi^{\varepsilon}_1(Z)}) = \mathsf{W}_1\Big(\mathcal{N}\big(0,e^{2\lambda}\big), \mathcal{N}\Big(\frac{\varepsilon}{\lambda}\big(e^{\lambda t}-1\big),e^{2\lambda}\Big)\Big) = \frac{\varepsilon}{\lambda}\big(e^{\lambda t}-1\big).
		\end{equation*} 
		On the other hand, consider two processes $(X^{(1)}_t)_{t \in [0,T]}, (X^{(2)}_t)_{t \in [0,T]}$ that satisfy
		\begin{align*}
			&\mathrm{d}X^{(1)}_t = b^{(1)}(	X^{(1)}_t, t) \; \mathrm{d}t + \sigma(t)\; \mathrm{d}B_t, \quad X^{(1)}_0 \sim \mathbb{P}\\
			&\mathrm{d}X^{(2)}_t = b^{(2)}(	X^{(2)}_t, t) \; \mathrm{d}t + \sigma(t)\; \mathrm{d}B_t, \quad X^{(2)}_0 \sim \mathbb{P},
		\end{align*}
		where $(B_t)_{[0, T]}$ is a Brownian motion. Under suitable assumptions on $\mathbb{P}, b^{(1)}, b^{(2)}$ and $\sigma$, see \cite{chen2023c, oko2023} and the references therein, one can use Girsanov's theorem to bound the total variation distance
		\begin{equation*}
			\operatorname{TV}(p^{(1)}_T, p^{(2)}_T) \lesssim \int_{0}^{T} \mathbb{E}[|b^{(1)}(X^{(1)}_t, t) - b^{(2)}(X^{(1)}_t, t)|^2]\; \mathrm{d}t.
		\end{equation*}
	\end{example}

In view of \Cref{eq:fmo}, any Lipschitz constraint on the class $\mathcal{M}$ should be aligned with the Lipschitz constant of $v.$ Looking at \eqref{marginal_prob_paths}, this Lipschitz constant depends on the construction of $v_t(\cdot|\cdot)$ and $p_t(\cdot|\cdot)$ and $p^*$. \\

The first term in \eqref{basic_orakel} is determined solely by the conditional vector fields $v_t(\cdot|y)$ and is therefore independent of the approximation class $\mathcal M$. To bound this term and to analyze the Lipschitz constant in \Cref{sec:lipschitz_constant}, we impose the following assumptions on the densities of the unknown distribution $\mathbb{P}^*$ and the latent distribution $\mathbb{U}$.
\begin{assumption}\label{ass:latent_distribution}$ $\newline \vspace{-0.5cm}
	\begin{enumerate}
		\item $d\geq 2, p^* \in B_{1, \infty}^{\alpha}, \alpha \in (0,1]$.
		\item 	We assume that
		\begin{equation*}
			p_t(x|y) \propto \exp\Big( -\frac{|x-\mu_t(y) |^2}{ 2 \sigma^2_{t}}\Big),
		\end{equation*}
		where the \textit{mean shift} $\mu\colon [0,1] \times \mathbb{R}^d \rightarrow \mathbb{R}$ and the \textit{variance function} $ \sigma \colon [0,1] \rightarrow \mathbb{R}_{>0} $ are smooth, component-wise monotone functions such that
		$
			\mu_0(y) = 0, \; \mu_1(y) = y, $ and $
			\sigma_0 = 1,  \sigma_1 = \sigma_{\min},
		$
		for a $\sigma_{\min} > 0$.
	\end{enumerate}
\end{assumption}

Under \Cref{ass:latent_distribution}, we know that 
\begin{equation}\label{eq:error_convolution}\mathsf{W}_1(\mathbb{P}^*, \mathbb{P}^{\psi_1(Z)})\lesssim \sigma_{\min}^{1+\alpha},\end{equation}
see proof of \citet[Theorem 3.3]{kunkel2025minimax}, where the first part of the proof does not use the assumption that the support of $\mathbb{P}^*$ is compact. Note that the constant in \eqref{eq:error_convolution} depends only on $\|p^*\|_{B^{\alpha}_{1, \infty}}$ and $d$.

\Cref{ass:latent_distribution} means in particular that the latent distribution is $\mathcal{N}(0, I_d)$. The monotonicity assumption guarantees that there is no unnecessary movement of mass. \Cref{ass:latent_distribution} coincides with the construction in \citet[Section 4]{lipman2023}. Very similar generative models based on probability Flows have been studied by \citet{gao2024, gao2024gaussian}. Additionally, the use of $\mathcal{N}(0, I_d)$ for the latent distribution is pervasive in generative modeling, e.g. in diffusions building up on \cite{sohl2015}. Often, $\mu$ and $\sigma$ are chosen as linear interpolations between the boundary conditions \cite{lipman2023, gao2024, liu2022}, this choice is far from unique, see e.g.\ the trigonometric interpolation \cite{albergo2022} or the related negative exponential quasi-interpolation in variance preserving diffusion modeling \citep{song2020}. In case the boundary conditions are met, we recover the same generated distribution $\mathbb{P}^{\psi_1(Z)}$, but the corresponding vector fields can differ significantly. In \Cref{sec:lipschitz_constant}, we are going to study the general setting.

\section{Lipschitz constant of the vector field}
Motivated by the upper bound in \eqref{basic_orakel}, we are going to study the Lipschitz constant of the vector field. Since $v$ is differentiable, our starting point is the Jacobian, which allows for several immediate conclusions. Based on the Jacobian, we derive assumptions on $\mathbb{P}^*$ that imply $\int_0^1 \Gamma_t\;\mathrm{d}t\leq C$. Afterwards, we study classes of unknown distributions that satisfy these assumptions or that admit no Lipschitz controlled vector field.

\label{sec:lipschitz_constant}
Under \Cref{ass:latent_distribution} \cite{lipman2023}, show that the corresponding vector field that generates $p_t(\cdot|y)$ for every $y \in \mathbb{R}^d$ is given by
\begin{equation} \label{eq:kernel_vector_field}
    v_t(x|y)= \frac{\sigma_t^{\prime}}{\sigma_t}(x-\mu_t(y))+\mu_t^{\prime}(y),
\end{equation}
 where $\sigma^{\prime}_t$ and $\mu^{\prime}_t$ are the derivatives in time. We conclude from this result that the definition of the variance function $\sigma_t$ is critical for bounding the Lipschitz constant, while only very special choices of the mean shift $\mu_t$ have an influence. Therefore, we will focus on polynomial choices of $\mu_t$, which generalize the choice of \cite{lipman2023}.
\begin{assumption}
    \label{ass:mu}
    We consider the following choices of $ \mu \colon [0,1] \times \mathbb{R}^d \rightarrow  \mathbb{R}^d$ for $\gamma >1$:
    \begin{equation*}
		 \mu_t(y) = t^{\gamma}y.
	\end{equation*} 
\end{assumption}
A natural first idea is to aim for a global Lipschitz constant, this corresponds to the assumption that $v$ has a bounded Lipschitz constant. The next result shows that such a constant cannot exist in this setting without considering properties of $\mathbb{P}^*$. We can upper and lower bound the Lipschitz constant of $v$.
\begin{theorem}\label{thm:Lipschitz_bound}
Let 
\begin{equation}\label{eq:jacobian}
B^t_{i,j} \coloneqq \sup_{x \in \mathbb{R}^d}\Big|\frac{\sigma_t^{\prime}}{\sigma_t} \mathds{1}_{i = j} + \big(\gamma t^{\gamma -1}- \frac{\sigma_t^{\prime}t^{\gamma}}{\sigma_t} \big)  \frac{t^{\gamma}}{\sigma_t^2} \operatorname{Cov}(Y^{x,t})_{ij}\Big|_{\infty},
\end{equation} where $Y^{x,t}$ is a random variable with density  $q \propto p_t(x|\cdot)p^*(\cdot )$ for $x \in \mathbb{R}^d$. Assume $B_{i,j}$ exists for all $i, j \in \{1,..,d\}$ and define $B \coloneqq (B_{ij})_{i,j = 1,...,d}$.
    For the Lipschitz constant in space $\Gamma_t $ of $v_t$ on $\mathbb{R}^d$, we have that
    \begin{equation*}
\max_{ij} B_{i,j}   \leq \Gamma_t \leq d  \max_{ij} B_{i,j}.
    \end{equation*} 
Additionally, there is a $t^* \in [0,1]$ such that 
 	$
 	\frac{\sigma_{t^*}^{\prime}}{\sigma_{t^*}} = \log(\sigma_{\min}),
 	$ for any choice of $\sigma_t$ satisfying \Cref{ass:latent_distribution}. Hence a global spatial Lipschitz bound on $v_t$ independent of $\sigma_{\min}$ and the covariance term is not feasible in this setting.   
    
\end{theorem}

The distribution with density $q$ is a weighted version of $p^*$ which shifts the probability mass to areas closer to $x$ relative to the rest of the mass of $\mathbb{P}^*$. The strength of this shift is governed by $\sigma_t$, smaller values correspond to a stronger weighing. The covariance depends on both, the covariance of $p^*$ and of $\sigma_t$. In case $x \in \operatorname{supp}(p^*)$, $q \overset{\sigma_{\min}\rightarrow 0}{\longrightarrow} \delta_x$. 

The upper and lower bound in \Cref{thm:Lipschitz_bound} depend on the choice of $\sigma_t$. We want to choose $\sigma_t$ such that $\big|\frac{\sigma_t^{\prime}}{\sigma_t}\big|$ does not grow too fast for small $\sigma_t$ or large $|\sigma_t^{\prime}|$. The next result shows 
that it is not possible to choose a function $\sigma_t$ such  that the integral over $\big|\frac{\sigma_t^{\prime}}{\sigma_t}\big|$ remains small. 
\begin{lemma}
    \label{thm:quotient_bound}
 For any choice of $\sigma_t$ satisfying \Cref{ass:latent_distribution},
    $
    \int_0^1 \big|\frac{\sigma_t^{\prime}}{\sigma_t} \big|\;\mathrm{d}t  = \log(\sigma_{\min}^{-1}).
   $

\end{lemma}

We note that \Cref{thm:quotient_bound} does not depend on the specific probability path, it holds for all quotients of smooth choices of $\sigma_t$. Furthermore, similar terms arise in all cases, where the chosen probability path implies vector fields of the form \eqref{eq:kernel_vector_field}.
 To compensate for $\frac{\sigma_t^{\prime}}{\sigma_t} = \log(\sigma_{\min})$ in the case of $i = j$, the second term in \eqref{eq:jacobian} must have the same absolute value. However, this prohibits a smaller bound when $i \neq j$.  \Cref{thm:quotient_bound} implies that, even the case of $\operatorname{Cov}(Y^{x,t})_{ij} \lesssim \sigma_t^2$ for all $i,j$ will lead to a logarithmic dependency on $\sigma_{\min}^{-1}$, which ultimately influences the rate. 

To control the Lipschitz constant, we thus need to assume that $p^*$ is such that the covariance term in \Cref{thm:Lipschitz_bound} decays in a controlled order for $\sigma_t \rightarrow \sigma_{\min}$ for $i = j$ and the same decay of the off-diagonal elements is fast enough.

\begin{assumption} \label{ass:covariance}
    Assume there is a $t^* \in \big[\frac{1}{2^{1/\gamma}},1\big)$ independent from $\sigma_{\min}$ such that for all $x \in \mathbb{R}^d$ 
    \begin{align*}
\mathrm{(I)}&           &    \operatorname{Cov}(Y^{x,t})_{ij}  & \lesssim  \Big(\frac{\sigma_t}{t^{\gamma}}\Big)^3,             &  i & \neq  j, & t > t^*,\\
\mathrm{(II)}&         &   \operatorname{Var}(Y^{x,t}_i)&=\Big(\frac{\sigma_t}{t^{\gamma}}\Big)^2  \Big(1+ O\Big(\Big(\frac{\sigma_t}{t^{\gamma}}\Big)^{\frac{1}{\kappa}}\Big) \Big),   & \text{f}&\text{or all } i, & t > t^*,\\
\mathrm{(III)}&   &   \operatorname{Cov}(Y^{x,t})_{ij}  & \leq C ,         &   \text{f}&\text{or all } i,j, & t \leq t^*,
\end{align*}
    where $Y^{x,t}$ is a random variable with density  $q \propto p_t(x|\cdot)p^*(\cdot )$ and $\kappa \in \mathbb{R}_{\geq 1}$ and $C$ are fixed constants. 
\end{assumption}

Under these assumptions, any choice of $\sigma_{t}$ satisfying \Cref{ass:latent_distribution} will lead to a bounded Lipschitz constant.
\begin{theorem} \label{thm:examples_sigma_t}  $ $
Grant \Cref{ass:latent_distribution}, \Cref{ass:mu} and \Cref{ass:covariance} with fixed parameters $\gamma, \kappa$ and $t^*$. Then there is a constant $C$ such that \begin{equation*}
        \int_0^1 \Gamma_t\; \mathrm{d}t \leq C.
        \end{equation*}
\end{theorem}

\Cref{ass:covariance} requires careful control over the covariance behavior of the weighted unknown distribution $q$. While the off-diagonal should decay fast enough, the variance must decay at a certain rate.
In the following, we present two classes of distributions that satisfy \Cref{ass:covariance}. One of the classes is a subset of the set of $\log$-concave densities, i.e.\ densities such that the logarithmic density is concave. The other setting shows that \Cref{ass:covariance} also allows for unbounded distribution that are not $\log$-concave and thus the Brascamp-Lieb inequality \citep[Theorem 4.1]{BRASCAMP1976} cannot be applied directly. The $\log$-concavity assumption has been popular in the analysis of Diffusion models and Flow Matching, see for example \citet{gao2024convergence, bruno2023diffusion, gao2025wasserstein, gao2024}. Additionally, we show that careful control over the variance is necessary, i.e.\ if the variance decays too fast, the vector field does not admit a controlled Lipschitz constant.

The first class of distributions are $\log$-concave distributions with $\operatorname{supp}(p^*) = \mathbb{R}^d$ whose potential is three times continuously differentiable and the partial derivatives are bounded, i.e.\ $p^*$ is such that there are $m, L > 0$ with
\begin{equation}\label{eq:assumption_p*new}
	p^*(x) = \exp(-V(x)), \quad V \in C^3, \quad \|V\|_{C^3}\leq L, \quad  V \gtrsim m I \quad  \forall x \in \mathbb{R}^d.
\end{equation} 

In case of a large choice $\sigma_{\min}$, which is not excluded at this point, $L$ should not be too large for technical reasons.

The second class of distributions is of the form \begin{equation}\label{eq:assumption_p*}
    p^*(x) \propto \exp\Big(- \frac{|x|^2}{2} - a (x)\Big), \quad  \|a\|_{C^2} =L < \infty.
\end{equation}
 This setting was used in \cite{stephanovitch2024smooth} to construct Lipschitz continuous pushforward maps in diffusion models. 
 We note that by \Cref{thm:p_besov} the first assumption in \Cref{ass:latent_distribution} is satisfied for densities of the form \eqref{eq:assumption_p*} with $\alpha = 1$. For simplicity, we choose $\gamma =1$, but the result extends easily to other $\gamma$.

\begin{proposition}\label{thm:meh} Set $\gamma = 1$.
  Assume that $\mathbb{P}^{*}$ is of the form \eqref{eq:assumption_p*new} or \eqref{eq:assumption_p*}. Then \Cref{ass:covariance} is fulfilled with \begin{equation*}
			\kappa = 1, \quad t^* \quad \text{ such that } \quad \sigma_{t^*} = c(L,d), \quad \text{ and } \quad  C =\frac{ 1}{m}\; (C = e^{2L} \text{ in case of \eqref{eq:assumption_p*}}),
		\end{equation*} where  $c(L,d)\in (0,1)$ is a constant that depends only on $L, d$ and whether $\mathbb{P}^{*}$ is of the form \eqref{eq:assumption_p*new} or \eqref{eq:assumption_p*}.
\end{proposition}
In case of  \Cref{eq:assumption_p*new}, the uniform covariance bound is proved via the Brascamp-Lieb inequality and the behavior of the covariance matrix is shown by a Taylor expansion on $\nabla V$.
The proof of the bound for distributions of the form \Cref{eq:assumption_p*} is mathematically more challenging: The uniform covariance bound is proven using the fact that we can control the effect of bounded perturbations of Gaussians on the variance, which follows from the Holley-Strooke perturbation principle. The bound on the diagonal entries employs proof techniques from the Cramer-Rao lower bound. The bound on the off-diagonal elements is developed from the Brascamp-Lieb type covariance estimate of \cite{menz2014brascamp}.

The proof of \Cref{thm:meh} reveals that for $t>t^*$ large enough and $i \neq j$, the covariance decay is of order $\sigma_t^4$ and $\kappa = 1$. This suggests that the class of distributions for which \Cref{thm:examples_sigma_t} applies allows for a slower decaying off-diagonal variance and a less controlled variance.

On the other hand, the next example demonstrates the consequences of a variance that decays too fast, which occurs even in case of a simple $\mathbb{P}^*$ such as the one-dimensional uniform distribution. This stresses the importance of a careful Lipschitz analysis.
\begin{example}\label{ex:uniform}
	Let $\gamma = 1, d = 1$ and $\mathbb{P}^* = \mathcal{U}([0,1])$. Then $\Gamma_t \gtrsim \big|\frac{\sigma_t^{\prime}}{\sigma_t} \big|$.
\end{example}

\Cref{ex:uniform} and \Cref{thm:quotient_bound} show that the exponential term appearing in \Cref{thm:error_decomp_easy} cannot be bounded without obtaining a factor $\sigma_{\min}^{-1}$. The proof of \Cref{ex:uniform} shows that the derivative of the vector field depends on $\operatorname{Var}_{Y \sim q}(Y)$, which decays arbitrarily fast for large $x$.

\section{Rate of convergence}\label{sec:rate}
 The goal of this chapter is to derive a rate of convergence, e.g.\ an upper bound on $\mathsf{W}_1(\mathbb{P}^*, \mathbb{P}^{\hat{\psi}_1(Z)})$ depending on the number of observations $n$, using the estimator $\hat{\psi}$ from \eqref{ODEhat}. For the set $\mathcal{M}$ we are going to use a subclass of ReLU networks, which will be defined below. Leveraging the insights from \Cref{sec:lipschitz_constant}, we consider unknown distributions $\mathbb{P}^*$ that admit Lipschitz controlled vector fields and a noise schedule $\sigma_t$ that stabilizes $v$ over all $t \in [0,1]$. We recover a classical trade-off in $\sigma_{\min}$, balancing this tuning parameter leads to the rate in \Cref{thm:rate_smooth}. 
	
	We start from \Cref{thm:error_decomp_easy}: In case the unknown distribution is such that \Cref{ass:covariance} holds and $p^* \in B^{\alpha}_{1, \infty}$, we can bound using \eqref{eq:error_convolution}
\begin{equation}\label{eq:variance_bias}
		\mathsf{W}_1(\mathbb{P}^*, \mathbb{P}^{\hat{\psi}_1(Z)}) \lesssim \sigma_{\min}^{1+ \alpha} +  \Big(\mathbb{E}_{\substack{t \sim \mathcal{U}[0,1]\\X_t \sim p_t}} \big[|v_t(X_t) - \tilde{v}_t(X_t) |^2\big]\Big)^{1/2}.
	\end{equation}
Since $\tilde{v}$ is chosen via \eqref{ODEhat} instead of \eqref{eq:fmo}, we cannot use the minimization property in this bound directly. Furthermore, the stability of $v_t$ depends on $\big|\frac{\sigma_t^{\prime}}{\sigma_t}\big|$, which deteriorates if $\sigma_{\min}$ is smaller. Thus, we expect a typical variance-bias trade-off in $\sigma_{\min}$ in \eqref{eq:variance_bias}. Note that the approach is different from the early-stopping approach natural to diffusions and classically adopted in Flow Matching analysis. Typically, instead of choosing $\sigma_{\min} >0$, this tuning parameter is set to $0$  and an early stopping time is imposed. While the two approaches are equivalent regarding the variance, the early stopping imposes an additional bias. This is commonly addressed with a very small early stopping time, see e.g.\ \citet[Lemma 4.3]{gao2024}, \citet[Lemma 11]{fukumizu2024}, which prevents the bound from capturing smoothness at this instance. The next theorem provides a more detailed decomposition of the former error in \eqref{eq:variance_bias} providing a classical classical oracle inequality. In order to merge the results later, we choose $\mu_t$ such that the results of \Cref{sec:lipschitz_constant} apply. For the variance function $\sigma_t$, the analysis in \Cref{sec:lipschitz_constant} reveals that we are not restricted to the linear case. Thus, we can choose a variance function that is suited for the application of a Bernstein-type inequality. Since Bernstein-type inequalities rely on absolute value bounds, we are going to choose $\sigma_t$ such that the absolute value of \eqref{eq:kernel_vector_field} is as small as possible. Our choice follows from \Cref{thm:quotient_bound} and the solution of the ODE $\frac{\sigma_t^{\prime}}{\sigma_t} = \log(\sigma_{{\min}})$ with the initial condition $\sigma_0=1.$ 

 \begin{assumption}\label{ass:mu_sigma_oracle_inequality}
     Assume that
     \begin{equation*}
         \mu_t(y) =    ty \quad \text{and} \quad  \sigma_t = (\sigma_{\min})^t.
     \end{equation*}
 \end{assumption}
 While the choice of $\sigma_t$ in this work is mainly for concentration convenience and to exploit the generality of \Cref{sec:lipschitz_constant}, \cite{tsimpos25a} show optimality for similar schedules in dynamic transport. Given the variety of concentration bounds available for linear and diffusion-type schedules, see \citet[Lemma 5.10]{gao2024}, \citet[Theorem 34]{azangulov2024}, \citet[Proposition 12]{stephanovitch2025generalization}, the below analysis can be extended to other choices of $\sigma_t$.\newline
First, we choose $\mathcal{M}$ as a set of continuous measurable functions $\tilde{v} \colon [0,1] \times \mathbb{R}^d \rightarrow \mathbb{R}$ such that for the $j$-th component function $\tilde{v}^j$ of $\tilde{v}$
\begin{equation}\label{eq:bound_functions_M}
|\tilde{v}^j|_{\infty}\leq e^{2L}\log(n)^3 + (1+e^{2L})\log(n)^2 + \log(n)+1.
\end{equation} 
We are going to justify this bound in the proof of the following theorem.

\begin{theorem}\label{thm:error_decomp} Let $p^*$ be of the form \eqref{eq:assumption_p*} and grant \Cref{ass:latent_distribution} and \Cref{ass:mu_sigma_oracle_inequality}. Assume that $\log(\sigma_{\min}^{-1})\sim \log(n)$. Then we have for every $a \in (0,1], \tau \in \mathbb{R}_{>0}, b>1$ and $n$ big enough with probability of $1-\frac{1}{n}$ that
\begin{align}
\mathbb{E}_{\substack{t \sim \mathcal{U}[0,1]\\X_t \sim p_t}} \big[|v_t(X_t) - \hat{v}_t(X_t) |^2\big] &\lesssim \inf_{\tilde{v}\in \mathcal{M}} \int \int | \tilde{v}_t(x) - v_t(x) |^2 p_t(x) \; \mathrm{d}x \; \mathrm{d}t  + n^{- \frac{1}{2}} \label{eq:error_decomposition} \\ &\quad  +\frac{a^{-1}\log(n)^7}{n} \log\big(2\mathcal{N}(\tau, g(\mathcal{M}), \| \cdot \|_{\infty}) \big) \notag + (2+a) \tau + a \log(n)^6 .
\end{align}
\end{theorem}

The proof of \Cref{thm:error_decomp} uses the Bernstein-type concentration inequality of \cite{chen2023b}, which is commonly used in the analysis of diffusion models, see for example \cite{yakovlev2025}.
\Cref{thm:error_decomp} shows that a careful choice of the variance shift can allow for to a logarithmic dependency of the oracle bounds on $\sigma_{\min}^{-1}$ instead of a linear dependency.

In practice, the parameterized function $\hat{v}$ is typically represented by a neural network. To analyze the theoretical performance in this setting, we connect \Cref{thm:error_decomp} with known approximation results for neural networks. Specifically, we consider the following fully connected feedforward networks with ReLU activation functions, which have been extensively studied, see, e.g., \cite{Yarotsky2017}, \cite{Guehring2020}, \cite{kohler2021}, \cite{Schmidt_Hieber2020}, \cite{suzuki2018}.

The Rectified Linear Unit (ReLU) activation is defined as $\phi\colon \mathbb{R} \rightarrow \mathbb{R}$, with
$
	\phi(x)\coloneqq\max (0, x).
$
For a vector $v=\left(v_1, \ldots, v_p\right) \in \mathbb{R}^p$, the shifted ReLU activation $\phi_v\colon \mathbb{R}^p \rightarrow \mathbb{R}^p$ is defined by
\begin{equation*}
\phi_v(x)\coloneqq\left(\phi\left(x_1-v_1\right), \ldots, \phi\left(x_p-v_p\right)\right), \quad x=\left(x_1, \ldots, x_p\right) \in \mathbb{R}^p.
\end{equation*}
A neural network with $L \in \mathbb{N}$ hidden layers and architecture $\mathcal{A}=\left(p_0, p_1, \ldots, p_{L+1}\right) \in \mathbb{N}^{L+2}$ is a function
\begin{equation}\label{def:NN}
f\colon \mathbb{R}^{p_0} \rightarrow \mathbb{R}^{p_{L+1}}, \quad f(x)\coloneqq W_L \circ \phi_{v_L} \circ W_{L-1} \circ \phi_{v_{L-1}} \circ \cdots \circ W_1 \circ \phi_{v_1} \circ W_0 \circ x,
\end{equation}
where $W_i \in \mathbb{R}^{p_{i+1} \times p_i}$ are weight matrices and $v_i \in \mathbb{R}^{p_i}$ are shift vectors. \\
Using networks of the form \eqref{def:NN} clipped at \eqref{eq:bound_functions_M} for the set $\mathcal{M}$, the next result combines the findings of \Cref{sec:lipschitz_constant}, \eqref{eq:error_convolution} and \Cref{thm:error_decomp} with approximation theory of \cite{Guehring2020} that allows for simultaneous approximation of a function and its derivative. 
To facilitate the transfer of results, we refrain from inserting $\alpha = 1$, which is the applicable smoothness in this setting as shown in \Cref{thm:p_besov}.

\begin{theorem} \label{thm:rate_smooth}
    Let $p^*$ be of the form \eqref{eq:assumption_p*} and grant \Cref{ass:latent_distribution} and \Cref{ass:mu_sigma_oracle_inequality}. Then for $n$ large enough with probability of $1-\frac{1}{n}$ for fixed $\eta>0$ 
\begin{equation*}
  \mathsf{W}_1(\mathbb{P}^*, \mathbb{P}^{\hat{\psi}_1(Z)}) \lesssim \operatorname{polylog}(n) n^{- \frac{1+\alpha}{d+ 4\alpha + 5 + \eta}},
\end{equation*}
    where $\hat{\psi}$ is the solution of an ODE whose vector field is given by a ReLU neural network, with no more than
   $ c \cdot \log \left(n\right)$ layers,  $ c \cdot n^{c(d, \alpha, \eta)}\cdot \log^2\left(n\right)$ nonzero weights, where $c$ and $c(d,\alpha, \eta)$ is a constant independent of $n$.

\end{theorem}
The proof of \Cref{thm:rate_smooth} exploits the fact that the vector field $v$ from \eqref{marginal_prob_paths} is by construction in $C^{\infty}$. 
In order to bound the higher order derivatives of $v$, we leverage the fact that densities of the form $\eqref{eq:assumption_p*}$ satisfy a $\log$-Sobolev inequality with a controlled constant. Interestingly, the dependence on the variance of $q$ from \Cref{thm:Lipschitz_bound} extends to a dependence on joint cumulants in higher order derivatives.
The rate in \Cref{thm:rate_smooth} benefits from smoothness in the unknown distribution. Since smoothing is an intrinsic property of Flow Matching, it is desirable that the rate also reflects this property. Even though the rate is not optimal, the gap to the optimal rate of $n^{- \frac{1+\alpha}{2\alpha +d}}$ is small for large $d$. It should be noted that this gap primarily results from the use of the general approximation result from \cite{Guehring2020}, which bounds the supremum norm error and not the weighted $L_2$ error appearing in \eqref{eq:error_decomposition}. This route reveals interesting properties of $v$, keeps the proof short and the paper focused on the distributional aspects of Flow Matching. A carefully fine tuned approximation result could close this gap.

Compared to \citet[Theorem 3.6]{kunkel2025minimax}, our result applies to networks with logarithmically growing depth and polynomially growing numbers of non-zero weights. This aligns the setting of the theorem more closely with practical applications, which is necessary to explain the empirical success of Flow Matching. This comes at the cost of a much stricter assumption on the unknown distribution and, at least when refraining from an tailored network construction, a slightly suboptimal rate. Compared to \cite{gao2024}, we gain in the rate since we are able to capture smoothness in $p^*$, while preserving a clean citation of the approximation result, but lose in the strictness of the evaluation metric ($\mathsf{W}_1$ instead of $\mathsf{W}_2$). 

It should however be noted, that the three results cannot be compared directly. \citet[Theorem 3.6]{kunkel2025minimax} assume that the support of $p^*$ is compact, whereas \Cref{thm:rate_smooth} assumes full support on $\mathbb{R}^d$. The distributions studied in \cite{gao2024} have full support, but the exact assumptions differ as well. Furthermore, they use a linear variance shift and an early stopping approach, which hinders direct comparison even more.

\section{Limitations and outlook}
Our analysis provides a theoretical explanation for the empirical success of Flow Matching and aligns more closely with real-world scenarios than previous results.  \Cref{sec:lipschitz_constant} paves the way for further research into broader classes of distributions whose vector fields have a bounded Lipschitz constant. The class of functions of form \eqref{eq:assumption_p*new} and \eqref{eq:assumption_p*} served as a toy example and is an interesting starting point for generalization. Conversely, the lower bound on the Lipschitz constant could be used to identify distributions that cannot be mimicked with a "good" rate of convergence, like indicated in \Cref{ex:uniform}.
While \Cref{ass:covariance} holds for a very broad range of variance functions, proofs based on concentration inequalities, such as \Cref{thm:rate_smooth}, depend on the specific choice of the variance function. Given optimal pushforward mappings, \cite{tsimpos25a} have investigated optimal noise schedules. Interestingly, their optimal schedule connects to the choice in \Cref{ass:mu_sigma_oracle_inequality}. Exploring optimal choices of variance functions from a statistical perspective would thus be very interesting.
Another promising direction is the introduction and consequences of an artificial Lipschitz control of the vector field. This influences the equality in \eqref{eq:fmo} as well as in the empirical counterpart. Controlling this effect could extend convergence results to larger classes of unknown distributions. A different, but interesting, question is whether there are Lipschitz controlled networks that approximate non-controlled vector fields and achieve good results. This requires a completely different approach in approximation theory. Another interesting aspect are bounds in different metrics. The bounds on the Jacobian in this work can serve as a starting point for analysis in the total variation distance, elaborating the ideas of \cite{su2025flow} combined with \cite{cai2025minimax, li2024sharp}.

\section{Proofs}\label{sec:proofs}

\subsection{Proofs of \Cref{sec:lipschitz_constant}}
\begin{proof}[Proof of \Cref{thm:Lipschitz_bound}]
First we show that we can calculate the Jacobian of $v$ explicitly. The proof and all subsequent proofs of auxiliary results are deferred to \Cref{sec:proofs_helper_lemma}.
\begin{lemma} \label{thm:jacobian}Fix $t \in [0,1].$
    The Jacobian with respect to $x$ of $v_t$ is given by
    \begin{equation}\label{eq:jacobian_form}
    D_x v_t (x)= \frac{\sigma_t^{\prime}}{\sigma_t} I_d + \Big(\gamma t^{\gamma-1}- \frac{\sigma_t^{\prime}t^{\gamma}}{\sigma_t} \Big)  \frac{t^{\gamma}}{\sigma_t^2} \operatorname{Cov}(Y^{x,t}).
    \end{equation}
\end{lemma}
The matrix $(B_{ij})_{i,j = 1,...,d}$ consists of the component wise supremum norm of $D_x v_t$.\\

By assumption $\operatorname{Cov}(Y^{x,t})_{ji}$ is bounded for all $x$.
     For the upper bound we use that by the mean value theorem, there exists an $\xi \in \mathbb{R}^d$ such that for the $i$-th coordinate function $v_t^{j}$ 
    \begin{equation*}
    |v^j_t(x)-v^j_t(y)| = \langle \nabla v^j_t(\xi), x-y \rangle \leq | v^j_t(\xi)||x-y| \leq \sqrt{d} |x-y|  \max_{i \in \{1,...,d \}} \Big\|\frac{\partial}{\partial x_i} v^j_t\Big\|_{\infty}.
\end{equation*}
Then
\begin{equation*}
 |v_t(x)-v_t(y)| = \Bigg|\left(\begin{array}{c}
      v^1_t(x)-v^1_t(y)  \\ \vdots \\ v^d_t(x)-v^d_t(y) 
 \end{array} \right)  \Bigg| \leq  d |x-y|  \max_{j \in \{1,...d \}} \max_{i \in \{1,...,d \}} \Big\|\frac{\partial}{\partial x_i} v^j_t\Big\|_{\infty} .
\end{equation*}
Therefore 
\begin{equation*}
    |v_t(x)-v_t(y)| \leq d |x-y| \max_{ij} \big\|\frac{\sigma_t^{\prime}}{\sigma_t} \mathds{1}_{i = j} +  \Big(\gamma t^{\gamma -1}- \frac{\sigma_t^{\prime}t^{\gamma}}{\sigma_t} \big) \frac{t^{\gamma}}{\sigma_t^2} \operatorname{Cov}(Y^{\cdot,t})_{ij}\Big\|_{\infty} =  d |x-y|  \max_{ij} B_{i,j}.
\end{equation*}
Hence $v_t$ is Lipschitz continuous.
For the lower bound we can use that by a Taylor expansion for $h, a \in \mathbb{R}^d$ 
\begin{equation}\label{eq:taylor}
    v_t(a+h)=v_t(a) + D_xv_t(a) h + r(h),
\end{equation}
with $
\lim_{|h|\rightarrow 0} \frac{|r(h)|}{|h|}  =0.
$
Let the smallest Lipschitz constant of $v_t$ be $\Gamma_t$. Using \eqref{eq:taylor} we can conclude
\begin{equation*}
|D_xv_t(a) h| = |v_t(a+h) - v_t(a) - r(h)| \leq  |v_t(a+h) - v_t(a) | + | r(h)|\\
 \leq \Gamma_t |h| + | r(h)|.
\end{equation*}
Now
\begin{equation*}
\|D_xv_t(a)\| = \underset{|h| \rightarrow 0}{\lim\operatorname{sup}}\,\frac{|D_xv_t(a) h|}{|h|} \leq \underset{|h| \rightarrow 0}{\lim\operatorname{sup}}\, \Gamma_t + \frac{| r(h)|}{|h|} = \Gamma_t. 
\end{equation*}
Let $ v^{i}_t$ be the $i$-th component function of $v_t$. Then for every $i \in \{1,...,d \}$ and every $j \in \{1,...,d  \}$
\begin{equation*}
    \sup_{a \in \mathbb{R}^d}\|D_xv_t(a)\|  = \sup_{a \in \mathbb{R}^d} \sup_{|w| = 1} |D_xv_t(a)w| \geq 
    \sup_{a \in \mathbb{R}^d} \sup_{|w| = 1} |e_i^{\top } D_xv_t(a)w| =  \sup_{a \in \mathbb{R}^d} \sup_{|w| = 1}   |\langle \nabla v^{i}_t(a), w \rangle |.
\end{equation*}
As the dual norm of the euclidean norm is the euclidean norm,
\begin{equation*}
    \sup_{a \in \mathbb{R}^d} \sup_{|w| = 1}   |\langle \nabla v^{i}_t(a), w \rangle | = \sup_{a \in \mathbb{R}^d}   |\nabla v^{i}_t(a)|\geq  \sup_{a \in \mathbb{R}^d}  |D_xv_t(a)_{ij}|.
\end{equation*}
Since $i,j$ were arbitrary and the entries of the Jacobian are of the form \eqref{eq:jacobian_form}, we obtain the bound on $\Gamma_t$.

For the existence of a $t^*$ such that $\frac{\sigma_{t^*}^{\prime}}{\sigma_t} = \log(\sigma_{\min}^{-1})$, we begin by setting
$
\frac{\sigma_t^{\prime}}{\sigma_t} = h_t,
$
where $h_t$ is a continuous function on $[0,1].$ By separations of variables, all of the solutions of this ODE are of the form 
$
\sigma_t = c e^{H_t},
$where $H_t$ is an anti-derivative of $h_t$ and $c \in \mathbb{R}$. we use $\sigma_0 = 1$ as initial condition, which leads to 
$
1 = ce^{H_0}$ which is equivalent to $c = \frac{1}{e^{H_0}}. 
$
To assure $\sigma_1 = \sigma_{\min}$, we need to choose $H$ such that
$
\sigma_{\min} = \frac{e^{H_1}}{e^{H_0}}$ which is equivalent to $  H_1 - H_0 = \log(\sigma_{\min}).
$
By the mean-value-theorem there is a $t^*\in [0,1]$ such that $H_{t^*}^{\prime} = h_{t^*} = \log(\sigma_{\min})$. 
\end{proof}

\begin{proof}[Proof of \Cref{thm:quotient_bound}]
   By change of variables we have that
$
\int_0^1 \big|\frac{\sigma_t^{\prime}}{\sigma_t} \big|\;\mathrm{d}t = - \int_0^1 \frac{\sigma_t^{\prime}}{\sigma_t} \;\mathrm{d}t =  \int_{\sigma_{\min}}^1 \frac{1}{u} \;\mathrm{d}u = \log(\sigma_{\min}^{-1}).
$\qedhere
\end{proof}

\begin{proof}[Proof of \Cref{thm:examples_sigma_t}]
 For small $t$, we can use the following simple bound, which is independent from the decay of $\operatorname{Cov}(Y^{x,t})$ in $t$:
\begin{lemma}\label{thm:easy_lipschitz}
    Let $t^*$ be such that $\sigma_{t^*} = \frac{1}{\vartheta}$ for $\vartheta \in \mathbb{R}_{\geq 1}$. Grant \Cref{ass:covariance} (III). Then
    \begin{equation*}\int_0^{t^*} \Gamma_t\; \mathrm{d}t \lesssim \vartheta^2 (1 + \log(\vartheta)).\end{equation*}
\end{lemma}

For large $t$, we need to assume that $\operatorname{Cov}(Y^{x,t})_{ij}$ decays fast enough for $i \neq j$ and $t \rightarrow 1$ to bound the integral over all $B_{ij}$ for $i \neq j$. Under \Cref{ass:covariance} (I), we know that for all $i \neq j$ and all $x \in \mathbb{R}^d$
\begin{align*}
	\int_{t^*}^1 \Big|\Big(\gamma t^{\gamma - 1} - \frac{\sigma_t^{\prime} t^{\gamma}}{\sigma_t}\Big)&\frac{t^{\gamma}}{\sigma_t^2} \operatorname{Cov}(Y^{x,t})_{ij}\Big| \; \mathrm{d}t  \leq (t^*)^{-2\gamma}\int_{t^*}^1  \big|(\gamma t^{\gamma -1}\sigma_t-\sigma_t^{\prime}t^{\gamma})\big|\;\mathrm{d}t \\&\leq  (t^*)^{-2\gamma} \gamma \int_0^1 \sigma_t -  (t^*)^{-2\gamma}\int_0^1 \sigma_t^{\prime} \;\mathrm{d}t=  (t^*)^{-2\gamma}( \gamma + 1-  \sigma_{\min}).
\end{align*}
For the bound of \begin{equation*}
	\int_{t^*}^1 \Big|\frac{\sigma_t^{\prime}}{\sigma_t}+  (\gamma t^{\gamma - 1} - \frac{\sigma_t^{\prime} t^{\gamma}}{\sigma_t})\frac{t^{\gamma}}{\sigma_t^2} \operatorname{Var}(Y^{x,t}_i)\Big| \; \mathrm{d}t
\end{equation*}
we need to use \Cref{ass:covariance} (II). Inserting the expression for the variance $\operatorname{Var}(Y^{x,t}_i)$ we obtain
\begin{align*}
&	\int_{t^*}^1 \Big|\frac{\sigma_t^{\prime}}{\sigma_t}+  (\gamma t^{\gamma - 1} - \frac{\sigma_t^{\prime} t^{\gamma}}{\sigma_t})\frac{t^{\gamma}}{\sigma_t^2} \operatorname{Var}(Y^{x,t}_i)\Big| \; \mathrm{d}t  = \int_{t^*}^1 \Big|\frac{\sigma_t^{\prime}}{\sigma_t}+  \Big(\gamma t^{\gamma - 1} - \frac{\sigma_t^{\prime} t^{\gamma}}{\sigma_t}\Big)\frac{t^{\gamma}}{\sigma_t^2} \Big(\frac{\sigma_t}{t^{\gamma}}\Big)^2  \Big(1+ O\Big(\Big(\frac{\sigma_t}{t^{\gamma}}\Big)^{\frac{1}{\kappa}}\Big) \Big)\Big| \; \mathrm{d}t\\
	& \leq \int_{t^*}^1 \Big|\frac{\sigma_t^{\prime}}{\sigma_t}+  \Big(\gamma t^{\gamma - 1} - \frac{\sigma_t^{\prime} t^{\gamma}}{\sigma_t}\Big)\frac{t^{\gamma}}{\sigma_t^2} \Big(\frac{\sigma_t}{t^{\gamma}}\Big)^2 \Big| \; \mathrm{d}t + \int_{t^*}^1 \Big| \Big(\gamma t^{\gamma - 1} - \frac{\sigma_t^{\prime} t^{\gamma}}{\sigma_t}\Big)\frac{t^{\gamma}}{\sigma_t^2} \Big(\frac{\sigma_t}{t^{\gamma}}\Big)^2   O\Big(\Big(\frac{\sigma_t}{t^{\gamma}}\big)^{\frac{1}{\kappa}}\Big) \Big| \; \mathrm{d}t.
\end{align*}
The first term simplifies to
\begin{equation*}
	\int_{t^*}^1 |\frac{\sigma_t^{\prime}}{\sigma_t}+  \Big(\gamma t^{\gamma - 1} - \frac{\sigma_t^{\prime} t^{\gamma}}{\sigma_t}\Big)\frac{t^{\gamma}}{\sigma_t^2} \Big(\frac{\sigma_t}{t^{\gamma}}\Big)^2 | \; \mathrm{d}t  = \int_{t^*}^1 |\gamma t^{-1}| \; \mathrm{d}t = - \gamma \log(t^*).
\end{equation*}
For the second term we obtain
\begin{align*}
&	\int_{t^*}^1 \Big| \Big(\gamma t^{\gamma - 1} - \frac{\sigma_t^{\prime} t^{\gamma}}{\sigma_t}\Big)\frac{t^{\gamma}}{\sigma_t^2} \Big(\frac{\sigma_t}{t^{\gamma}}\Big)^2   O\Big(\Big(\frac{\sigma_t}{t^{\gamma}}\Big)^{\frac{1}{\kappa}}\Big) \Big| \; \mathrm{d}t \leq \int_{t^*}^1 \Big| \Big(\gamma t^{\gamma - 1} - \frac{\sigma_t^{\prime} t^{\gamma}}{\sigma_t}\Big)\Big| \Big|\frac{1}{t^{\gamma}} O\Big(\Big(\frac{\sigma_t}{t^{\gamma}}\Big)^{\frac{1}{\kappa}}\Big) \Big| \; \mathrm{d}t\\
	& \lesssim \int_{t^*}^1 \gamma t^{-1 - \frac{\gamma}{\kappa}} \sigma_t^{\frac{1}{\kappa}}\; \mathrm{d}t - \int_{t^*}^1 \frac{\sigma_t^{\prime}}{\sigma_t} \frac{\sigma_t^{\frac{1}{\kappa}}}{t^{\frac{\gamma}{\kappa}}}\; \mathrm{d}t \leq \gamma  \int_{t^*}^1 t^{-1 - \frac{\gamma}{\kappa}}\; \mathrm{d}t -  (t^*)^{-\frac{\gamma}{\kappa}} \int_{t^*}^1 \frac{\sigma_t^{\prime}}{\sigma_t} \sigma_t^{\frac{1}{\kappa}}\; \mathrm{d}t .
\end{align*}
Bounding the last term with $\kappa( (t^*)^{-\frac{\gamma}{\kappa}}-1) +  (t^*)^{-\frac{\gamma}{\kappa}} \kappa \Big(\sigma_{1}^{\frac{1}{\kappa}} - \sigma_{t^*}^{\frac{1}{\kappa}}\Big)$ and using that $\sigma_{1}^{\frac{1}{\kappa}} - \sigma_{t^*}^{\frac{1}{\kappa}} \leq 1$ yields the following bound on the integral over the Lipschitz constant
\begin{equation*}
	\int_0^1 \Gamma_t \; \mathrm{d}t \lesssim \vartheta^2 (1 + \log(\vartheta)) +(t^*)^{-2\gamma}(\gamma +1 )+ \kappa( (t^*)^{-\frac{\gamma}{\kappa}}-1) +  (t^*)^{-\frac{\gamma}{\kappa}} \kappa.\qedhere
\end{equation*}
\end{proof}

\textit{Preparations for the proof of \Cref{thm:meh}.}\\
For the proof we need the following additional definitions:
\begin{definition}\label{Poincaré_inequality}
    A distribution $\mathbb{P}^*$ on $\mathbb{R}^d$ with density $p^*$ satisfies the Poincaré inequality with Poincaré constant $\rho > 0$ if for all smooth functions $f$ such that the terms below are well-defined
    \begin{equation*}
        \operatorname{Var}(f(X))\coloneqq\int\left(f(x) - \mathbb{E}[f(X)]\right)^2 p^*(x)\mathrm{d} x \leq \frac{1}{\rho}\mathbb{E}[|\nabla f(X)|^2].
    \end{equation*}
\end{definition}
From \cite[Theorem 3.20]{Boucheron2013} we know that $\mathcal{N}(0, I_d)$ satisfies the Poincaré inequality with Poincaré constant $1.$ Thus,  via defining $g(x) \coloneqq f(\sigma x + \mu)$ and 
$
    \operatorname{Var}(g(X)) \leq \sigma^2 \mathbb{E}[|\nabla g(X)|^2],
$ we see that the Poincaré constant of $\mathcal{N}(\mu, \sigma^2 I_d)$ is $\sigma^{-2}$.
		\begin{definition}\label{def:log_sobolev}
	A distribution $\mathbb{P}^*$ on $\mathbb{R}^d$ with density $p^*$ satisfies the logarithmic Sobolev inequality with $\log$-Sobolev constant $\lambda > 0$ if for all smooth functions $f$
	\begin{equation}\label{eq:moment_bound_1}
		\operatorname{Ent}(f^2(X))\coloneqq \mathbb{E}[f^2(X)\log(f^2(X))] - \mathbb{E}[f^2(X)]\log(\mathbb{E}[f^2(X)]) \leq \frac{2}{\lambda}\mathbb{E}[|\nabla f(X)|^2].
	\end{equation}
		The standard Gaussian satisfies the  $\log$-Sobolev inequality with $\log$-Sobolev constant $\lambda =1$ \cite[Theorem 5.4]{Boucheron2013}. Analogously to Poincaré inequality, its easy to see that $\mathcal{N}(\mu, \sigma^2I_d)$ satisfies the $\log$-Sobolev inequality with $\log$-Sobolev constant $\lambda =\sigma^{-2}$.
\end{definition}
Now we are ready to prove \Cref{thm:meh}.
\begin{proof}[Proof of \Cref{thm:meh}]$ $\newline
\underline{Property (III), $\mathbb{P}^*$ of form \eqref{eq:assumption_p*new}:}	\newline
We start by defining $\Phi (y) = \frac{|x-t^{\gamma}y|^2}{2 \sigma_t^2} + V(y)$, such that $q(y) \propto \exp(-\Phi(y))$. For every $t$, we have that $\nabla^2\Phi(y) = \frac{t^{\gamma}}{\sigma_t^2} I + \nabla^2V(y)\gtrsim  \nabla^2V(y) \gtrsim m I $. Thus $\mathbb{P}^*$ is strongly $\log$-concave. By \cite{BRASCAMP1976}, we can bound
\begin{equation*}
	\operatorname{Cov}(Y^{x,t})_{ij} \leq \sqrt{\operatorname{Var}(Y^{x,t}_i)}\sqrt{\operatorname{Var}(Y^{x,t}_j)} \leq \frac{1}{m}.
\end{equation*}
\underline{Property (I) and (II), $\mathbb{P}^*$ of form \eqref{eq:assumption_p*new}:}	\newline
We choose $t^*$ such that $\frac{\sigma_t^2}{t^{2\gamma}}L\leq \min(\frac{1}{2}, \frac{\sigma_t}{t^{\gamma}})$. Our goal is to derive a representation of $\Sigma$, the covariance matrix of $q$. 
Using the divergence theorem, as common in multivariate Stein's identity type results relating to \cite{Stein1972}, see e.g.\ \citet[Section 2.2.2]{oates2017control}, we obtain for the smooth function $f(y) = y-\mathbb{E}[Y]$ and $q \propto \exp(-\Psi)$
\begin{equation*}
	\mathbb{E}[\nabla f(Y)] = \mathbb{E}[f(Y)(\nabla\Psi(Y))^{\top}]\quad  \Longleftrightarrow \quad I = \mathbb{E}[(Y-\mathbb{E}[Y])(\nabla\Psi(Y))^{\top}].
\end{equation*}
For $t > t^*>0$, we define $\mu \coloneqq \frac{x}{t^{\gamma}}$ and $s = \frac{\sigma_t}{t^{\gamma}}$ such that $\Psi(y) = \frac{|y-\mu|^2}{2s^2} + V(y)$. Then since $\mathbb{E}[Y - \mathbb{E}[Y]] = 0$
\begin{equation*}
	\mathbb{E}[(Y-\mathbb{E}[Y])(\nabla\Psi(Y))^{\top}] = \mathbb{E}\Big[(Y-\mathbb{E}[Y])\Big(\frac{Y-\mu}{s^2} + \nabla V(Y)\Big)^{\top}\Big] = \frac{1}{s^2} \Sigma +  \mathbb{E}\big[(Y-\mathbb{E}[Y]) (\nabla V(Y))^{\top}\big].
\end{equation*}
Thus
\begin{equation*}
	\Sigma = s^2 I - s^2  \mathbb{E}\big[(Y-\mathbb{E}[Y]) (\nabla V(Y))^{\top}\big].
\end{equation*}
A multivariate Taylor expansion yields
\begin{equation*}
	\nabla V(Y) = \nabla V(\mathbb{E}[Y]) + \nabla^2 V(\mathbb{E}[Y])(Y- \mathbb{E}[Y]) + R(Y),
\end{equation*}
where $R$ depends on the bound of the third partial derivatives and $|R(Y)|\lesssim |Y- \mathbb{E}{Y}|^2$. Again, since $\mathbb{E}[Y - \mathbb{E}[Y]] = 0$, we obtain
\begin{equation*}
	\Sigma =  s^2 I -  s^2 \Sigma \nabla^2( V(\mathbb{E}[Y]))^{\top} - R^{\prime}, 
\end{equation*}
where $|R^{\prime}|\lesssim \mathbb{E}[|Y - \mathbb{E}[Y]|^3]$. By \citet[Theorem 21.2, Remark 12.4]{Villani2008}, who refers to \cite{Bakry1985}, which is in French, a strongly $\log$-concave function satisfies the logarithmic Sobolev equation. Using Herbst's argument, e.g.\ along the lines of \cite{Ledoux1999}, a standard consequence of the logarithmic Sobolev equation is that the distribution is sub-Gaussian with a parameter antiproportional to the parameter of the bound of the Hessian. For $t > t^*$, we have that
\begin{equation*}
	\nabla^2 \Psi(y) = \frac{1}{s^2}I + \nabla^2V(y) \gtrsim \Big(\frac{1}{s^2} - L\Big)I \gtrsim \frac{1}{2s^2}I.
\end{equation*}
Note that we used a different property of $V$, namely the bound on the spectral norm of $\nabla^2V$. From the sub-Gaussianity, we can control the moment $ \mathbb{E}[|Y - \mathbb{E}[Y]|^3] \lesssim s^3$, see \citet[Lemma 1]{jin2019short}. Since $\nabla^2 V(\mathbb{E}[Y])$ is invertible, we obtain
\begin{equation*}
	\Sigma = s^2(I - R^{\prime})(I + s^2 (\nabla V(\mathbb{E}[Y]))^{\top})^{-1}.
\end{equation*} Next, we use a Neumann series, which is admissible since $\|s^2 (\nabla V(\mathbb{E}[Y]))^{\top}\| \leq s^2L \leq s \leq 1$ to rewrite
\begin{equation*}
	(I + s^2 (\nabla V(\mathbb{E}[Y]))^{\top})^{-1} = I - s^2(\nabla V(\mathbb{E}[Y]))^{\top} + \sum_{k = 2}^{\infty}(-s^2 (\nabla V(\mathbb{E}[Y]))^{\top})^k.
\end{equation*}
For the sum we have using a geometric series
\begin{equation*}
	\| \sum_{k = 2}^{\infty}(-s^2 (\nabla V(\mathbb{E}[Y]))^{\top})^k\| \leq  \sum_{k = 2}^{\infty} \|s^2 (\nabla V(\mathbb{E}[Y]))^{\top})\|^k \leq \sum_{k = 2}^{\infty} s^k = \frac{s^2}{1-s} \leq s^2.
\end{equation*}
Thus
\begin{equation*}
	\Sigma_{ij} = s^2 (\mathds{1}_{i = j} + O(s^2)).
\end{equation*}

\underline{Property (III), $\mathbb{P}^*$ of form \eqref{eq:assumption_p*}:}\newline
For the uniform bound, we use that the distribution with density
\begin{equation}
    q \propto \exp\Big(-\frac{|x-ty|^2}{2 \sigma_t^2}-  \frac{|y|^2}{2}\Big) =\exp\Big( -\frac{1}{2} \Big( 1+ \frac{t^2}{\sigma_t^2}\Big) \Big(|y|^2- \Big\langle \frac{t}{\sigma_t^2} \frac{x}{1+ \frac{t^2}{\sigma_t^2}}, y \Big\rangle \Big)\Big) \label{eq:density_gaussian}
\end{equation}
is a Gaussian distribution with variance $ \big( 1+ \frac{t^2}{\sigma_t^2}\big)^{-1} I_d$. Hence the density defined above satisfies the Poincaré inequality with constant $\big( 1+ \frac{t^2}{\sigma_t^2}\big)$. 
 Using the Holley-Strooke perturbation principle \cite{Holley1987}, in the form of \citet[Lemma 1.2]{Ledoux2001}, we can bound the Poincaré constant $\rho$ of the perturbed Gaussian via 
$
\rho \geq e^{- 4L} \big( 1+ \frac{t^2}{\sigma_t^2}\big).
$
Thus
\begin{equation*}
\operatorname{Var}(Y^{x,t}_{i}) \leq \Big(  e^{-4L} \Big( 1+ \frac{t^2}{\sigma_t^2}\Big) \Big)^{-1} \leq \frac{e^{4L}}{1+ \frac{t^2}{\sigma_t^2}} \leq e^{4L},
\end{equation*}
with $L$ from \eqref{eq:assumption_p*}. Using \begin{equation*}
\operatorname{Cov}(Y^{x,t}_{i}, Y^{x,t}_{j})\leq \sqrt{\operatorname{Var}(Y^{x,t}_{i})\operatorname{Var}(Y^{x,t}_{j})},
\end{equation*} we conclude that for all $t \in [0,1]$
$
\operatorname{Cov}(Y^{x,t}_{i}, (Y^{x,t}_{j}) \leq e^{4L}.
$
Hence we can set $C =  e^{4L}$.\newline
\underline{Property (II), $\mathbb{P}^*$ of form \eqref{eq:assumption_p*}:}\newline
    For the variance, we use that
    \begin{equation*}
    \operatorname{Var}(Y^{x,t}_{i}) = \mathbb{E}[(Y^{x,t}_{i})^2] - \mathbb{E}[Y^{x,t}_{i} ]^2.
\end{equation*}
    Let $\varphi$ denote the density of $\mathcal{N}(0,I_d)$. Further, let $t>0$. Then
    \begin{equation*}
        \mathbb{E}[Y^{x,t}_{i}] = \frac{\int y_i \varphi\big(\frac{x-ty}{\sigma_{t}} \big) p^*(y)\; \mathrm{d}y }{\int \varphi\big(\frac{x-ty}{\sigma_{t}} \big) p^*(y)\; \mathrm{d}y } =\frac{\int \big(\frac{x_i-z_i\sigma_t}{t} \big)\varphi(z) p^*(\frac{x-z\sigma_t}{t})\; \mathrm{d}y }{\int \varphi(z) p^*(\frac{x-z\sigma_t}{t})\; \mathrm{d}y  }
    = \frac{x_i}{t} - \frac{\sigma_t}{t} \frac{\int z_i\varphi(z) p^*(\frac{x-z\sigma_t}{t})\; \mathrm{d}y }{\int \varphi(z) p^*(\frac{x-z\sigma_t}{t})\; \mathrm{d}y  }.
    \end{equation*}
    Similarly
    \begin{align*}
        \mathbb{E}[\big(Y^{x,t}_{i}\big)^2] & =  \frac{\int \big(\frac{x_i-z_i\sigma_t}{t} \big)^2\varphi(z) p^*(\frac{x-z\sigma_t}{t})\; \mathrm{d}y }{\int \varphi(z) p^*(\frac{x-z\sigma_t}{t})\; \mathrm{d}y  }
        \big(\frac{x_i}{t}\big)^2 -  \frac{2x_i\sigma_t}{t^2} \frac{\int z_i\varphi(z) p^*(\frac{x-z\sigma_t}{t})\; \mathrm{d}y }{\int \varphi(z) p^*(\frac{x-z\sigma_t}{t})\; \mathrm{d}y  } + \big(\frac{\sigma_t}{t}\big)^2 \frac{\int z_i^2\varphi(z) p^*(\frac{x-z\sigma_t}{t})\; \mathrm{d}y }{\int \varphi(z) p^*(\frac{x-z\sigma_t}{t})\; \mathrm{d}y  }.
    \end{align*}
Now define 
\begin{equation*}
A(z_i) \coloneqq \frac{\int z_i\varphi(z) p^*(\frac{x-z\sigma_t}{t})\; \mathrm{d}z}{\int\varphi(z) p^*(\frac{x-z\sigma_t}{t})\; \mathrm{d}z}, \quad A(z_i^2) \coloneqq \frac{\int z_i^2\varphi(z) p^*(\frac{x-z\sigma_t}{t})\; \mathrm{d}z}{\int\varphi(z) p^*(\frac{x-z\sigma_t}{t})\; \mathrm{d}z}.
\end{equation*}
We obtain for the variances
\begin{align}
    \mathbb{E}[(Y^{x,t}_{i})^2]-\mathbb{E}[Y^{x,t}_{i}]^2 & = \Big(\frac{x_i}{t}\Big)^2 -  \frac{2x_i\sigma_t}{t^2} A(z_i) + \Big(\frac{\sigma_t}{t}\Big)^2 A(z_i^2) - \Big(\frac{x_i}{t} - \frac{\sigma_t}{t} A(z_i)\Big)^2  \Big(\frac{\sigma_t}{t}\Big)^2 \Big(A(z_i^2) - A(z_i)^2\Big). \label{eq:variance_decomp_A}
\end{align}
Hence we need to bound the component variances of a random variable $Z$ with density
\begin{align}
  p_Z(z) &= \frac{\varphi(z) p^*(\frac{x-z\sigma_t}{t})}{\int\varphi(z) p^*(\frac{x-z\sigma_t}{t})\; \mathrm{d}z}=\frac{ \exp\big(-\frac{|z|^2}{2} - \frac{|x-z\sigma_t|^2}{2t^2} - a(\frac{x-z\sigma_t}{t})\big)}{\int\exp\big(-\frac{|z|^2}{2} - \frac{|x-z\sigma_t|^2}{2t^2} - a(\frac{x-z\sigma_t}{t})\big)\; \mathrm{d}z}\notag \\
  & = \frac{ \exp\big(-\frac{1}{2} \big(1+ \frac{\sigma_t^2}{t^2} \big) \big(|z|^2 - 2 \langle z, x \frac{\sigma_t}{t^2 (1+ (\frac{\sigma_t}{t})^2)}\rangle \big) - a(\frac{x-z\sigma_t}{t})\big)}{\int\exp\big(-\frac{1}{2} \big(1+ \frac{\sigma_t^2}{t^2} \big) \big(|z|^2 - 2 \langle z, x \frac{\sigma_t}{t^2 (1+ (\frac{\sigma_t}{t})^2)}\rangle \big) - a(\frac{x-z\sigma_t}{t})\big)\; \mathrm{d}z} .\label{eq:gauss_form_pZ}
\end{align}
First we bound the influence of the perturbation function $a$ on the expected value. To do so, we note that 
\begin{equation*}
    \mathbb{E}_{p_Z}[\partial_{z_i}\log(p(Z))]= \int \partial_{z_i}p(z)\; \mathrm{d}z = \int ... \int_{z_i \in \mathbb{R} }\partial_{z_i}p(z) \; \mathrm{d}z_i \; \mathrm{d}(z_1,...,z_{i-1}, z_{z+1}, ...,z_d) = 0.
\end{equation*}
This implies for \eqref{eq:gauss_form_pZ} that
\begin{align*}
   & 0 = - \mathbb{E}\Big[\Big(1+ \frac{\sigma_t^2}{t^2} \Big)  \Big(Z_i -   x_i \frac{\sigma_t}{t^2 (1+ (\frac{\sigma_t}{t})^2)} \Big) - \frac{\sigma_t}{t} \partial_{z_i}a\Big(\frac{x-Z\sigma_t}{t}\Big)  \Big]\\
   \Longleftrightarrow \quad & \Big(1+ \frac{\sigma_t^2}{t^2} \Big)  \Big(\mathbb{E}[Z_i] -   x_i \frac{\sigma_t}{t^2 (1+ (\frac{\sigma_t}{t})^2)} \Big)  =  \mathbb{E}\Big[ \frac{\sigma_t}{t} \partial_{z_i}a\Big(\frac{x-Z\sigma_t}{t}\Big)  \Big]\\
    \Longleftrightarrow \quad &\Big|\mathbb{E}[Z_i] -   x_i \frac{\sigma_t}{t^2 (1+ (\frac{\sigma_t}{t})^2)} \Big| = \Bigg|\frac{\frac{\sigma_t}{t}\mathbb{E}\Big[  \partial_{z_i}a(\frac{x-Z\sigma_t}{t})\Big]}{\big(1+ \frac{\sigma_t^2}{t^2} \big)}\Bigg|.
\end{align*}
Hence \begin{equation}\label{eq:variance_bound_0}
\Big|\mathbb{E}[Z_i] -   x_i \frac{\sigma_t}{t^2 (1+ (\frac{\sigma_t}{t})^2)} \Big|  \leq \frac{\sigma_t}{t} \frac{L}{\big(1+ \frac{\sigma_t^2}{t^2} \big)},
\end{equation} where $L$ is from \eqref{eq:assumption_p*}. From \eqref{eq:gauss_form_pZ} we can see that this bounds the influence of the perturbation function $a$ on the expected value.
Now we bound the variance. To do so, we first note that
\begin{equation}
    \mathbb{E}_{p_Z}[(\partial_{z_i}\log(p(Z)))^2]= \int \frac{(\partial_{z_i }p_Z(z))^2}{ p_Z(z)} \; \mathrm{d}z. \label{eq:variance_bound_1}
\end{equation}
As
\begin{equation*}
    \partial_{z_i} \log(p_Z(z)) = \frac{\partial_{z_i}p_Z(z)}{p_Z(z)}, \quad \partial^2_{z_i} \log(p_Z(z)) = \frac{\partial^2_{z_i}p_Z(z)}{p_Z(z)} - \frac{(\partial_{z_i}p_Z(z))^2}{(p_Z(z))^2},
\end{equation*} we have for \eqref{eq:variance_bound_1}
\begin{equation*}
    \mathbb{E}_{p_Z}[(\partial_{z_i}\log(p(Z)))^2]= \mathbb{E}_{p_Z}\Big[\frac{\partial^2_{z_i}p_Z(z)}{p_Z(z)}\Big] - \mathbb{E}_{p_Z}\big[ \partial^2_{z_i} \log(p_Z(Z)) \big]. 
\end{equation*} Since $\partial_{z_i}p_Z(z) \rightarrow 0$ for $|z| \rightarrow \infty$, we conclude
\begin{equation}\label{eq:variance_bound_2}
    \mathbb{E}_{p_Z}[(\partial_{z_i}\log(p(Z)))^2]=  - \mathbb{E}_{p_Z}[ \partial^2_{z_i} \log(p_Z(Z)) ]. 
\end{equation} 
For the right hand side, we obtain
\begin{equation*}
    -\mathbb{E}_{p_Z}[ \partial^2_{z_i} \log(p_Z(Z)) ]  = \Big(1+ \Big(\frac{\sigma_t}{t}\Big)^2 \Big) + \Big(\frac{\sigma_t}{t}\Big)^2 \mathbb{E}_{p_Z}\Big[ \partial^2_{z_i} a (\frac{x-\sigma_t z}{t})\Big].
\end{equation*}
The left side of \eqref{eq:variance_bound_2} can be rewritten as
\begin{align*}
     &\mathbb{E}_{p_Z}[(\partial_{z_i}\log(p(Z)))^2]  =  \mathbb{E}_{p_Z}\Big[\Big(  \Big(1+ \frac{\sigma_t^2}{t^2}\Big)\Big(Z_i - x_i \frac{\sigma_t}{t^2 (1+ (\frac{\sigma_t}{t})^2)}\Big) + \frac{\sigma_t}{t} \partial_{z_i} a \Big(\frac{x- \sigma_t Z}{t}\Big)\Big)^2\Big]\\
      &=  \mathbb{E}_{p_Z}\Big[\Big(  \Big(1+ \frac{\sigma_t^2}{t^2}\Big)\Big(Z_i - x_i \frac{\sigma_t}{t^2 (1+ (\frac{\sigma_t}{t})^2)}\Big)\Big)^2\Big]   +2 \mathbb{E}_{p_Z}\Big[  \Big(1+ \frac{\sigma_t^2}{t^2}\Big)\Big(Z_i - x_i \frac{\sigma_t}{t^2 (1+ (\frac{\sigma_t}{t})^2)}\Big)  \frac{\sigma_t}{t} \partial_{z_i} a \Big(\frac{x- \sigma_t Z}{t}\Big)\Big] \\ & \quad + \Big(\frac{\sigma_t}{t} \Big)^2\mathbb{E}_{p_Z}\Big[\Big( \partial_{z_i} a (\frac{x- \sigma_t Z}{t})\Big)^2\Big].
\end{align*}
Then 
\begin{align*}
 &   \mathbb{E}_{p_Z}\Big[\Big(  \Big(1+ \frac{\sigma_t^2}{t^2}\Big)\Big(Z_i - x_i \frac{\sigma_t}{t^2 (1+ (\frac{\sigma_t}{t})^2)}\Big)\Big)^2\Big]  = \mathbb{E}_{p_Z}\Big[\Big(  \Big(1+ \frac{\sigma_t^2}{t^2}\Big)\Big(Z_i - \mathbb{E}[Z_i] + \mathbb{E}[Z_i]- x_i \frac{\sigma_t}{t^2 (1+ (\frac{\sigma_t}{t})^2)}\Big)\Big)^2\Big]\\
     & =\Big(1+ \frac{\sigma_t^2}{t^2}\Big)^2 \mathbb{E}_{p_Z}\big[  (Z_i - \mathbb{E}[Z_i])^2\big]+ 2\Big(1+ \frac{\sigma_t^2}{t^2}\Big)^2  \Big( \mathbb{E}[Z_i]- x_i \frac{\sigma_t}{t^2 (1+ (\frac{\sigma_t}{t})^2)}\Big)\mathbb{E}_{p_Z}\big[  (Z_i - \mathbb{E}[Z_i])\big]\\
      & + \Big(1+ \frac{\sigma_t^2}{t^2}\Big)^2\Big(  \mathbb{E}[Z_i]- x_i \frac{\sigma_t}{t^2 (1+ (\frac{\sigma_t}{t})^2)}\Big)^2 = \Big(1+ \frac{\sigma_t^2}{t^2}\Big)^2 \Big(\mathbb{E}_{p_Z}\big[  (Z_i - \mathbb{E}[Z_i])^2\big]+ \Big(  \mathbb{E}[Z_i]- x_i \frac{\sigma_t}{t^2 (1+ (\frac{\sigma_t}{t})^2)}\Big)^2\Big).
\end{align*}
Additionally
\begin{align*}
     &\mathbb{E}_{p_Z}\Big[  \Big(1+ \frac{\sigma_t^2}{t^2}\Big)\Big(Z_i - x_i \frac{\sigma_t}{t^2 (1+ (\frac{\sigma_t}{t})^2)}\Big)  \frac{\sigma_t}{t} \partial_{z_i} a \Big(\frac{x- \sigma_t Z}{t}\Big)\Big]  \\&=     \mathbb{E}_{p_Z}\Big[  \Big(1+ \frac{\sigma_t^2}{t^2}\Big)\Big(Z_i - \mathbb{E}[Z_i] + \mathbb{E}[Z_i] - x_i \frac{\sigma_t}{t^2 (1+ (\frac{\sigma_t}{t})^2)}\Big)  \frac{\sigma_t}{t} \partial_{z_i} a \Big(\frac{x- \sigma_t Z}{t}\Big)\Big]\\
    & = \Big(1+ \frac{\sigma_t^2}{t^2}\Big)\frac{\sigma_t}{t} \mathbb{E}_{p_Z}\Big[  (Z_i - \mathbb{E}[Z_i])  \partial_{z_i} a \Big(\frac{x- \sigma_t Z}{t}\Big)\Big]   +   \Big(1+ \frac{\sigma_t^2}{t^2}\Big)\frac{\sigma_t}{t}\Big( \mathbb{E}[Z_i] - x_i \frac{\sigma_t}{t^2 (1+ (\frac{\sigma_t}{t})^2)}\Big)  \mathbb{E}_{p_Z}\Big[ \partial_{z_i} a \Big(\frac{x- \sigma_t Z}{t}\Big)\Big].
\end{align*}
Hence
\begin{align*}
 & \Big(1+ \frac{\sigma_t^2}{t^2}\Big)^2   \mathbb{E}_{p_Z}\Big[\Big(  Z_i - \mathbb{E}[Z_i]\Big)^2\Big]  = \Big(1+ \Big(\frac{\sigma_t}{t}\Big)^2 \Big) + \Big(\frac{\sigma_t}{t}\Big)^2 \mathbb{E}_{p_Z}\Big[ \partial^2_{z_i} a \Big(\frac{x-\sigma_t z}{t}\Big)\Big] - \Big(\frac{\sigma_t}{t} \Big)^2\mathbb{E}_{p_Z}\Big[\Big( \partial_{z_i} a (\frac{x- \sigma_t Z}{t})\Big)^2\Big] \\ &\quad  - \Big(1+ \frac{\sigma_t^2}{t^2}\Big)^2\Big(  \mathbb{E}[Z_i]- x_i \frac{\sigma_t}{t^2 (1+ (\frac{\sigma_t}{t})^2)}\Big)^2 -2\Big(1+ \frac{\sigma_t^2}{t^2}\Big)\frac{\sigma_t}{t} \mathbb{E}_{p_Z}\Big[  (Z_i - \mathbb{E}[Z_i])  \partial_{z_i} a \Big(\frac{x- \sigma_t Z}{t}\Big)\Big] \\
 & \quad-  2 \Big(1+ \frac{\sigma_t^2}{t^2}\Big)\frac{\sigma_t}{t}\Big( \mathbb{E}[Z_i] - x_i \frac{\sigma_t}{t^2 (1+ (\frac{\sigma_t}{t})^2)}\Big)  \mathbb{E}_{p_Z}\Big[ \partial_{z_i} a \Big(\frac{x- \sigma_t Z}{t}\Big)\Big].
\end{align*}
Combining this with \eqref{eq:variance_bound_0} and the fact that $a$ has bounded derivatives, we obtain the following upper and lower bounds: 
\begin{equation*}
    \Big(1+ \frac{\sigma_t^2}{t^2}\Big)^2  \mathbb{E}_{p_Z}\big[\big(  Z_i - \mathbb{E}[Z_i]\big)^2\big]  \leq \Big(1+ \Big(\frac{\sigma_t}{t}\Big)^2 \Big) + \Big(\frac{\sigma_t}{t}\Big)^2 L + 2\Big(1+ \frac{\sigma_t^2}{t^2}\Big)\frac{\sigma_t}{t} L \sqrt{\mathbb{E}_{p_Z}\big[  (Z_i - \mathbb{E}[Z_i])^2 \big]}  + 2  \Big(1+ \frac{\sigma_t^2}{t^2}\Big)\frac{\sigma_t^2}{t^2}  \frac{L^2}{\big(1+ \frac{\sigma_t^2}{t^2} \big)},  
\end{equation*}
and
\begin{equation}\label{eq:variance_bound_3}
\begin{aligned}
     \Big(1+ \frac{\sigma_t^2}{t^2}\Big)^2   \mathbb{E}_{p_Z}\Big[\Big(  Z_i - \mathbb{E}[Z_i]\Big)^2\Big] & \geq \Big(1+ \Big(\frac{\sigma_t}{t}\Big)^2 \Big) - (\frac{\sigma_t}{t})^2L - \Big(\frac{\sigma_t}{t} \Big)^2L^2  - \Big(1+ \frac{\sigma_t^2}{t^2}\Big)^2\Big( \frac{\sigma_t}{t} \frac{L}{\Big(1+ \frac{\sigma_t^2}{t^2}\Big)}\Big)^2\\
& \quad -2\Big(1+ \frac{\sigma_t^2}{t^2}\Big)\frac{\sigma_t}{t} L \sqrt{\mathbb{E}_{p_Z}\big[  (Z_i - \mathbb{E}[Z_i])^2 \big]} -  2 \Big(1+ \frac{\sigma_t^2}{t^2}\Big)\frac{\sigma_t^2}{t^2}  \frac{L^2}{\big(1+ \frac{\sigma_t^2}{t^2}\big)}.
\end{aligned}
\end{equation}
With the same reasoning via the Poincaré constant of a Gaussian and the Holley-Stroock perturbation principle, we conclude that
$
\mathbb{E}_{p_Z}\big[  (Z_i - \mathbb{E}[Z_i])^2 \big]\leq e^{4L}.
$
As $\big(1+ \frac{\sigma_t^2}{t^2}\big)^2>0$ and
$
\frac{1}{1+ \frac{\sigma_t^2}{t^2}} = 1- \frac{\frac{\sigma_t^2}{t^2}}{1+ \frac{\sigma_t^2}{t^2}},
$
dividing the first term in the upper and lower bound loosens the bound further. Thus we obtain
$
\mathbb{E}_{p_Z}\big[\big(  Z_i - \mathbb{E}[Z_i]\big)^2\big] =  1 + O(\frac{\sigma_t}{t}).
$
Multiplying with $\frac{\sigma_t^2}{t}$ yields the result for the variances via \eqref{eq:variance_decomp_A} $
   \operatorname{Var}(Y^{x,t}_i) =   \big(\frac{\sigma_t}{t}\big)^2  \big(1+ O\big(\frac{\sigma_t}{t}\big) ).
$\newline
\underline{Property (I), $\mathbb{P}^*$ of form \eqref{eq:assumption_p*}:}\newline
For the covariances, we need to find a stricter upper bound. We are going to use \citet[Theorem 2.3]{menz2014brascamp} applied to the component functions $f_i \colon \mathbb{R}^d \rightarrow \mathbb{R}, \; f_i(y)=y_i$. First we verify that the assumptions are fulfilled in our setting. By construction, the distribution of $Y_i^{x,t}$ is a bounded perturbation of a Gaussian. Hence Assumption 2.2 in \citet{menz2014brascamp} is satisfied.

For the Poincaré constant of the $i$-th conditional measure, e.g. the Poincaré constant of the distribution with the density
\begin{equation*}
p(y_i|y_1,...,y_{i-1}, y_{i+1}, ..., y_d) \propto \exp\Big(- \frac{(x_i-ty_i)^2}{2 \sigma_t^2}- \frac{y_i^2}{2} - a(y_i|y_1,...,y_{i-1}, y_{i+1}, ..., y_d)\Big),
\end{equation*}
where $a(\cdot |y_1,...,y_{i-1}, y_{i+1}, ..., y_d) \coloneqq a(y_1,...,y_{i-1}, \cdot,y_{i+1}, ..., y_d)$ and $i \in \{1,...,d \}$, we can use the same arguments as in the proof of property (III) to obtain
$
\rho_i^t \geq e^{-4L}(1+ \frac{t^2}{\sigma_t^2}).
$
Note that $a(\cdot |y_1,...,y_{i-1}, y_{i+1}, ..., y_d)$ is still bounded.
Furthermore, we know that the off-diagonal entries of the Hessian of the $\log$-density are bounded by $L$. To profit from easier notation later, we define the matrix
\begin{equation*}
A_t = (A_{t, ij})_{i,j = 1,...,d}, \quad \text{where} \quad A_{t, ij} \coloneqq \begin{cases}
    \rho^t_i - e^{-4L}\frac{t^2}{\sigma_t^2}, & i = j,\\
    -L, & i\neq j.
\end{cases}
\end{equation*} 
We need to find a bound on $t^*$ such that the matrix $A_{t} + e^{-4L}\frac{t^{2 }}{\sigma_t^2} I_d$ is positive definite. As $A_{t}$ is symmetric by construction, we know that the eigenvalues of $A_{t} + e^{-4L}\frac{t^{2 }}{\sigma_t^2} I_d$ are real numbers. By Gerschgorins theorem \cite[Satz 2]{Geschgorin1931}, we know that the eigenvalues of $A_{t} +e^{-4L} \frac{t^{2 }}{\sigma_t^2} I_d$ are in the following union of intervals:
\begin{equation}\label{eq:gerschgorin_circle}
D = \bigcup_{i = 1}^d \Big[\rho_i^t  - L(d-1), \rho_i^t  +L(d-1)\Big].
\end{equation}
Inserting the lower bound of $\rho_i^t$, we choose $t^*$ such that for all $t<t^*$
\begin{equation}\label{eq:bound_t*_1}
e^{-4L}\Big(1+ \frac{t^2}{\sigma_t^2}\Big) > L(d-1) \quad \Longleftarrow \quad \frac{\sigma_t}{t}< \frac{1}{\sqrt{e^{4L}L(d-1)-1}}.
\end{equation}
 
 For this choice of $t^*$, we conclude that the matrix $A_{t,ij} + e^{-4L}\frac{t^{2 }}{\sigma_t^2} I_d$ is positive definite. Thus all of the assumptions in \cite[Theorem 2.3]{menz2014brascamp} are satisfied. 
 \newline
We conclude that 
$
 |\operatorname{Cov}(Y^{x,t})_{ij}| \leq ((e^{-4L}\frac{t^{2}}{\sigma_t^2}I_d + A_t)^{-1})_{ij}.
$
Now \begin{equation}\label{eq:property_3}
   \Big(e^{-4L}\frac{t^{2}}{\sigma_t^2}I_d + A_t\Big)^{-1} = e^{4L} \frac{\sigma_t^2}{t^{2}} \Big(I_d + e^{4L}\frac{\sigma_t^2}{t^{2}} A_t \Big)^{-1}.
\end{equation}
As $A_t$ is a symmetric matrix, the spectral norm is the absolute value of the largest eigenvalue. Using Gerschgorins theorem \cite[Satz II]{Geschgorin1931} again,  we know that the eigenvalues of $A_{t}$ are in the following union of intervals:
\begin{equation}
D = \bigcup_{i = 1}^d \Big[\rho_i^t  - e^{-4L}\frac{t^{2}}{\sigma_t^{2}} - L(d-1), \rho_i^t  - e^{-4L}\frac{t^{2}}{\sigma_t^{2}} +L(d-1)\Big].
\end{equation}
As $
\rho_i^t  - e^{-4L}\frac{t^{2}}{\sigma_t^{2}} = e^{-4L}
$
we conclude that the largest eigenvalue is bounded by
$
	 \lambda_{\max}(A_t) \leq e^{-4L}+ L(d+1).
$
We therefore deduce that
\begin{equation*}
    \Big\|e^{4L} \frac{\sigma_t^2}{t^{2}} A_t \Big\| =    e^{4L}\frac{\sigma_t^2}{t^{2}}  \Big\|  A_t \Big\| \leq  \frac{\sigma_t^2}{t^{2}} (1 +e^{4L} L(d-1))  .
\end{equation*}
If we choose $t^*$ such that for $t>t^*$
\begin{equation*}
 \frac{\sigma_t^2}{t^{2}} (1 +e^{4L} L(d-1)) \leq  \frac{\sigma_t}{t} < 1 \quad  \Longleftrightarrow \quad \frac{\sigma_t}{t} \leq \frac{1}{1+e^{4L}L(d-1)} ,
\end{equation*}
we can use a Neumann series. This gives
\begin{equation*}
    \Big(I_d + e^{4L}\frac{\sigma_t^2}{t^{2}} A_t \Big)^{-1} = \sum_{k = 0}^{\infty} (-1)^k e^{4Lk}\frac{\sigma_t^{2k}}{t^{2 k}} A_t^k = I_d - e^{4L}\frac{\sigma_t^2}{t^{2}} A_t +  \sum_{k = 2}^{\infty} (-1)^k e^{4Lk}\frac{\sigma_t^{2k}}{t^{2 k}} A_t^k.
\end{equation*}
Now for the $ij$-th element, we obtain
\begin{equation*}
    \Big(I_d + e^{4L}\frac{\sigma_t^2}{t^{2}} A_t \Big)^{-1}_{ij} \leq \Big(\mathds{1}_{i = j} +  e^{4L}\frac{\sigma_t^2}{t^{2}} A_{t,ij} + \Big\|  \sum_{k = 2}^{\infty} (-1)^k e^{4Lk}\frac{\sigma_t^{2k}}{t^{2 k}} A_t^k\Big\| \Big).
\end{equation*}
Then we can bound 
   \begin{equation*}
       \Big\|  \sum_{k = 2}^{\infty} (-1)^k e^{4Lk}\frac{\sigma_t^{2k}}{t^{2 k}} A_t^k\Big\|  \leq   \sum_{k = 2}^{\infty}  \Big\|e^{4Lk}\frac{\sigma_t^{2k}}{t^{2 k}} A_t^k\Big\|  \leq  \sum_{k = 2}^{\infty}  \Big\|e^{4L}\frac{\sigma_t^{2}}{t^{2 }} A_t\Big\|^k \leq \sum_{k = 2}^{\infty} \big( \frac{\sigma_t}{t}\big)^{k}
   \end{equation*} 
 Using the convergence of the geometric series, we get that
$
\sum_{k = 2}^{\infty} \big( \frac{\sigma_t}{t}\big)^{k} = \frac{\frac{\sigma_t^2}{t^2}}{1- \frac{\sigma_t}{t}}\leq \frac{\sigma_t^2}{t^2}.
$
Inserting everything into \eqref{eq:property_3}, we obtain
   \begin{equation*}
 \Big(I_d + \frac{\sigma_t^2}{t^{2}} A_t \Big)^{-1}_{ij}  \lesssim \frac{\sigma_t^2}{t^{2}}(\mathds{1}_{i = j} +  \frac{\sigma_t^2}{t^2} ).
   \end{equation*}
If we choose $t^*\geq \frac{1}{2}$, we know that $\frac{\sigma_t^2}{t^{2}} \leq 2 \sigma_t^2$. This gives the smaller bound on the covariances. 
\end{proof}
\begin{proof}[Proof of \Cref{ex:uniform}]
	We start combining \eqref{marginal_prob_paths}, \eqref{eq:kernel_vector_field} and the choice of $\mathbb{P}^*$ to obtain
	\begin{equation*}
		v_t(x) = \frac{A(x)}{B(x)}, \quad \text{where} \quad A(x)\coloneqq 	\int_0^1 \Big(\frac{\sigma_t^{\prime}}{\sigma_t}(x-ty)+y \Big)\exp\Big(- \frac{(x-ty)^2}{2 \sigma_t^2}\Big)\; \mathrm{d}y, \quad B(x) = \int_0^1 \exp\Big(- \frac{(x-ty)^2}{2 \sigma_t^2}\Big)\; \mathrm{d}y.
	\end{equation*}
	By the quotient rule, we obtain
$
		\partial_x 	v_t(x) = \frac{\partial_x A(x)}{B(x)} - \frac{\partial_x B(x)}{B(x)}\frac{A(x)}{B(x)}.
$
	Differentiating leads to
	\begin{align*}
		\partial_x A(x) &= \int_0^1 \frac{\sigma_t^{\prime}}{\sigma_t} \exp\Big(- \frac{(x-ty)^2}{2 \sigma_t^2}\Big)\; \mathrm{d}y - \int_0^1 \Big(\frac{\sigma_t^{\prime}}{\sigma_t}(x-ty)+y \Big)\frac{x-ty}{\sigma_t^2}\exp\Big(- \frac{(x-ty)^2}{2 \sigma_t^2}\Big)\; \mathrm{d}y,\\
		\partial_x B(x) & = -  \int_0^1\frac{x-ty}{\sigma_t^2}\exp\Big(- \frac{(x-ty)^2}{2 \sigma_t^2}\Big)\; \mathrm{d}y.
	\end{align*}
	We note that the integrands on the right side are finite due to the exponential decay, which makes the exchange of integration and differentiation valid. Inserting the partial derivatives and using $Y^x \sim  p_t(x|\cdot)p^*(\cdot)$
	\begin{equation*}
		\partial_x 	v_t(x) =  \frac{\sigma_t^{\prime}}{\sigma_t} - \mathbb{E}\Big[\big(  \frac{\sigma_t^{\prime}}{\sigma_t} (x-tY^x)+Y^x\big)\frac{x-tY^x}{\sigma_t^2}\Big] + \mathbb{E}\Big[  \frac{x-tY^x}{\sigma_t^2}\Big]\mathbb{E}\Big[\Big(  \frac{\sigma_t^{\prime}}{\sigma_t} (x-tY^x)+Y^x\Big)\Big].
	\end{equation*}
	Rewriting the second term as the covariance and using linearity and the irrelevance constants of the covariance leads to 
	\begin{equation*}
		\operatorname{Cov}\Big(\Big(  \frac{\sigma_t^{\prime}}{\sigma_t} (x-tY^x)+Y^x\Big), \frac{x-tY^x}{\sigma_t^2}\Big)= \Big(\frac{\sigma_t^{\prime}}{\sigma_t^3}t^2 - \frac{t}{\sigma_t^2}\Big)\operatorname{Var}(Y^x).
	\end{equation*}
	Next we show that for a fixed $\varepsilon>0$, we can always find a $x \in \mathbb{R}$ such that $\operatorname{Var}(Y^x)\leq 2\varepsilon^2$. Since the support of $\mathbb{P}^*$ is $[0,1]$, we can bound
	\begin{equation*}
		\operatorname{Var}(Y^x) = \mathbb{E}[(Y^x- \mathbb{E}[Y^x])^2]\leq \mathbb{E}[(Y^x- 1)^2] = \mathbb{E}[(Y^x- 1)^2\mathds{1}_{\{Y \geq 1-\varepsilon\}}] +  \mathbb{E}[(Y^x- 1)^2\mathds{1}_{\{Y < 1-\varepsilon\}}] \leq \varepsilon^2 + \mathbb{P}(Y^x < 1-\varepsilon).
	\end{equation*}
	Then
	\begin{equation*}
		\mathbb{P}(Y < 1-\varepsilon) = \frac{\int_0^{1-\varepsilon}\exp(- \frac{(x-ty)^2}{2 \sigma_t^2})\; \mathrm{d}y}{\int_0^{1}\exp(- \frac{(x-ty)^2}{2 \sigma_t^2})\; \mathrm{d}y} \leq  \frac{\int_0^{1-\varepsilon}\exp(- \frac{(x-t(1-\varepsilon))^2}{2 \sigma_t^2})\; \mathrm{d}y}{\int_{1- \varepsilon/2}^{1}\exp(- \frac{(x-t(1-\varepsilon/2))^2}{2 \sigma_t^2})\; \mathrm{d}y} = \frac{2(1-\varepsilon)}{\varepsilon} \exp\Big(- \frac{xt\varepsilon - t^2(\varepsilon- 3\varepsilon^2/4)}{2 \sigma_t^2}\Big).
	\end{equation*}
	Thus, for every $\varepsilon$, there is a $x\in \mathbb{R}$ such that $\mathbb{P}(Y < 1-\varepsilon) \leq \varepsilon^2$. Hence, for every $\varepsilon>0$, we obtain
	\begin{equation*}
		\partial_x 	v_t(x) =  \frac{\sigma_t^{\prime}}{\sigma_t} + \Big(\frac{\sigma_t^{\prime}}{\sigma_t^3}t^2 - \frac{t}{\sigma_t^2}\Big) O(\varepsilon^2) \quad \Longrightarrow \quad \sup_{x \in \mathbb{R}}|\partial_x 	v_t(x) | \gtrsim \Big| \frac{\sigma_t^{\prime}}{\sigma_t} \Big| \quad \Longrightarrow \quad \Gamma_t \gtrsim \Big| \frac{\sigma_t^{\prime}}{\sigma_t} \Big|. \qedhere
	\end{equation*}

\end{proof}

\subsection{Proofs of \Cref{sec:rate}}

\begin{proof}[Proof of \Cref{thm:error_decomp}]
        To validate the bound on $|v_t^j|_{\infty}$, we insert $\log(\sigma_{\min})\sim \log(n)$ a bit later.
We start by defining $A \coloneqq [-\log(n), \log(n)]^d$ and
\begin{equation}
    g\colon \mathcal{M} \times \mathbb{R}^d \rightarrow \mathbb{R}^d, \quad g(v, y) \coloneqq \int_0^1 \int |v_t (x) - v_t(x|y)|^2 p_t(x|y)\; \mathrm{d}x \; \mathrm{d}t.
\end{equation}
Then we use \Cref{thm:equivalence_constant},
\begin{equation}\label{eq:error_decomp_1}
    \mathbb{E}_{\substack{t \sim \mathcal{U}[0,1]\\X_t \sim p_t}} \big[|v_t(X_t) - \hat{v}_t(X_t) |^2\big] = \mathbb{E}_{\substack{t\sim \mathcal{U}[0,1], \\Y \sim p^*, \\X_t \sim p_t\left(\cdot | Y\right)}}\big[\left|\hat{v}_t(X_t)-v_t\left(X_t| Y\right)\right|^2\big] - C, \; \text{with} \; C \coloneqq \mathbb{E}_{\substack{t\sim \mathcal{U}[0,1], \\Y \sim p^*, \\X_t \sim p_t\left(\cdot | Y\right)}}[|v_t(X_t)-v_t(X_t|Y)|^2].
\end{equation}

We can split the integral
\begin{equation*}
\mathbb{E}_{\substack{t\sim \mathcal{U}[0,1], \\Y \sim p^*, \\X_t \sim p_t\left(\cdot | Y\right)}}\big[\left|\hat{v}_t(X_t)-v_t\left(X_t| Y\right)\right|^2\big]  = \mathbb{E}_{Y \sim \mathbb{P}^*}[g(\hat{v}, Y)\mathds{1}_{Y \in A}] + \mathbb{E}_{Y \sim \mathbb{P}^*}[g(\hat{v}, Y)\mathds{1}_{Y \in \mathbb{R}^d \setminus A}]
\end{equation*}

We have for every $x \in \mathbb{R}^d$ and $j \in \{1,...,d \}$ and the $t$-th component function $v_t^j$ of $v_t$
\begin{align*}
|v_t^j(x)| &= \Big| \int v^j_t(x|y) \frac{p_t(x|y)}{\int p_t(x|z)p^*(z) \; \mathrm{d}z}p^*(y) \; \mathrm{d}y \Big| = \Big| \int \Big(\frac{\sigma^{\prime}_t}{\sigma_t} x_j + \Big(1- \frac{\sigma^{\prime}_t}{\sigma_t} t \Big)y_j  \Big) \frac{p_t(x|y)}{\int p_t(x|z)p^*(z) \; \mathrm{d}z}p^*(y) \; \mathrm{d}y \Big|\\&  \leq \Big|\frac{\sigma^{\prime}_t}{\sigma_t}\Big| |x_j|  +\Big|\Big(1- \frac{\sigma^{\prime}_t}{\sigma_t} t \Big)\Big| \frac{\int |y_j| \exp(- \frac{|x-ty|^2}{2 \sigma_t^2} - \frac{|y|^2}{2} - a(y))\; \mathrm{d}y}{\int \exp(- \frac{|x-ty|^2}{2 \sigma_t^2} - \frac{|y|^2}{2} - a(y))\; \mathrm{d}y}\\&  \leq \Big|\frac{\sigma^{\prime}_t}{\sigma_t}\Big| |x_j| + e^{2L}\Big|\Big(1- \frac{\sigma^{\prime}_t}{\sigma_t} t \Big)\Big| \mathbb{E}_{Z \sim \mathcal{N}(\frac{t}{t^2 + \sigma_t^2}x, (1+ \frac{t^2}{\sigma_t^2})^{-1} I_d)}[|Z_j|]\\
& \leq \Big( \Big|\frac{\sigma^{\prime}_t}{\sigma_t}\Big| + e^{2L} \Big|\Big(1- \frac{\sigma^{\prime}_t}{\sigma_t} t \Big)\Big|\frac{t}{t^2 + \sigma_t^2}\Big)|x_j| +  \Big|\Big(1- \frac{\sigma^{\prime}_t}{\sigma_t} t \Big)\Big|\Big(1+ \frac{t^2}{\sigma_t^2}\Big)^{-\frac{1}{2}}.
\end{align*}
As
$
\max_{t \in [0,1]} \big|\frac{\sigma^{\prime}_t}{\sigma_t}\big| = \log(\sigma_{\min}^{-1})$, $(1+ \frac{t^2}{\sigma_t^2})^{-\frac{1}{2}} \leq 1
$
and due to \Cref{thm:helper_lemma}
$
\frac{t}{t^2 + \sigma_t^2} \leq \max(\log(\sigma_{\min}^{-1}), e^2)
$
we obtain for the maximum over $t \in [0,1]$ for $n$ big enough
$
    |v_t^j(x)|   \lesssim  \log(n)^2 |x_j| + \log(n).
$
The constants lead to the bound in \eqref{eq:bound_functions_M}.
An analogous calculation shows that, for $t\in[0,1]$, the norm of $v_t$ is bounded by
\begin{equation}\label{eq:bound_v_t}
    |v_t(x)|   \lesssim  \log(n)^2 |x| + \log(n).
\end{equation}
Therefore, we obtain for every $v \in \mathcal{M}$ and $y \in A$ using the construction of $v_t(\cdot|\cdot)$ and fact that the functions in $\mathcal{M}$ are cut at \eqref{eq:bound_functions_M} $[-\log(n), \log(n)]^d$ for all $t$, 
\begin{equation*}
    g(v, y) = \int_0^1 \int |v_t (x) - v_t(x|y)|^2 p_t(x|y)\; \mathrm{d}x \; \mathrm{d}t  \lesssim \log(n)^6 + \int_0^{1} \int |v(x|y)|^2 p_t(x|y)\; \mathrm{d}x \; \mathrm{d}t.
\end{equation*}
Since
\begin{equation*}
\int |v(x|y)|^2 p_t(x|y)\; \mathrm{d}x =\int \Big|\frac{\sigma^{\prime}_t}{\sigma_t} x + \Big(1- \frac{\sigma^{\prime}_t}{\sigma_t} t\Big) y  \Big|^2p_t(x|y)\; \mathrm{d}x= \Big|\frac{\sigma^{\prime}_t}{\sigma_t}\Big|^2\int |x|^2p_t(x|y)\; \mathrm{d}x + \Big(1- \frac{\sigma^{\prime}_t}{\sigma_t} t  \Big)^2 |y|^2, 
    \end{equation*}
and as $p_t(\cdot |y)$ is the density of a Gaussian with mean $y$ and variance $\sigma_t^2 I_d$, which implies
$
    \int |x|^2p_t(x|y)\; \mathrm{d}x  = t^2|y|^2 + \sigma_t^2 d,
$
    we obtain for $g(v,y)$
    \begin{equation*}
        g(v,y)  \lesssim  \log(n)^6 + \int_0^1\Big|\frac{\sigma^{\prime}_t}{\sigma_t}\Big|^2 (t^2|y|^2 + \sigma_t^2 d)\; \mathrm{d}t  +  |y|^2\int_0^1\Big(1- \frac{\sigma^{\prime}_t}{\sigma_t} t  \Big)^2 \; \mathrm{d}t \lesssim  \log(n)^6.
    \end{equation*}
We denote the constant in the bound of $g(v,y)$ by $D$.
Now we can use the Bernstein-type concentration inequality from \citet[Lemma 15]{chen2023b} and conclude that for every $a \in (0,1], \delta_1 \in (0,\frac{1}{3})$ and $\tau \in \mathbb{R}_{>0}$
\begin{align*}
    \mathbb{P} \Big( \sup_{\bar{v} \in \mathcal{M}}& \mathbb{E}_{Y \sim \mathbb{P}^*}[g(\bar{v}, Y)\mathds{1}_{Y \in A}] - \frac{1+a}{n} \sum_{i =1}^n g(\bar{v}, X_i)\mathds{1}_{X_i \in A}\\ & \quad  > \frac{(1+ 6/a) D\log(n)^6}{3 n} \log \Big( \frac{\mathcal{N}(\tau, g(\mathcal{M}), \|\cdot \|_{\infty})}{\delta_1}\Big) + (1+a) \tau\Big)\leq \delta_1.
\end{align*}
We keep the term $\frac{a}{n} \sum_{i = 1}^n  g(\bar{v}, X_i)\mathds{1}_{X_i \in A}$ separate. For $\frac{1}{n} \sum_{i = 1}^n  g(\bar{v}, X_i)\mathds{1}_{X_i \in A}$ we conclude that with a probability of at least $1-\delta_1$
\begin{align*}
 \int_A \int \int &|\hat{v}_t(x)-v_t(x|y)|^2 p_t(x|y)\; \mathrm{d}x \; \mathrm{d}t \;p^*(x)\;\mathrm{d}y\\ &  \leq \frac{1+a}{n} \sum_{i =1}^n g(\hat{v}, X_i)\mathds{1}_{X \in A}+\frac{(1+ 6/a) D\log(n)^6}{3 n}  \log \Big( \frac{\mathcal{N}(\tau, g(\mathcal{M}), \|\cdot \|_{\infty})}{\delta_1}\Big) + (1+a) \tau.
\end{align*}
Due to the choice of $\hat{v}$ as the empirical risk minimizer, we know that for every $\tilde{v} \in \mathcal{M}$
\begin{align*}
&\frac{1}{n} \sum_{i =1}^n g(\hat{v}, X_i) = \frac{1}{n} \sum_{i = 1}^n \Big( \int \int |\hat{v}_t(x)- v_t(x|X_i)|^2 p_t(x|X_i)\;\mathrm{d}x \; \mathrm{d}t\Big)\mathds{1}_{X_i \in A} \\ &=   \frac{1}{n} \sum_{i = 1}^n \int \int |\hat{v}_t(x)- v_t(x|X_i)|^2 p_t(x|X_i)\;\mathrm{d}x \; \mathrm{d}t-\frac{1}{n} \sum_{i = 1}^n \Big(\int \int |\hat{v}_t(x)- v_t(x|X_i)|^2 p_t(x|X_i)\;\mathrm{d}x \; \mathrm{d}t \Big)\mathds{1}_{X_i \notin A}\\ 
& \leq   \frac{1}{n} \sum_{i = 1}^n \int \int |\tilde{v}_t(x)- v_t(x|X_i)|^2 p_t(x|X_i)\;\mathrm{d}x \; \mathrm{d}t -\frac{1}{n} \sum_{i = 1}^n \Big(\int \int |\hat{v}_t(x)- v_t(x|X_i)|^2 p_t(x|X_i)\;\mathrm{d}x \; \mathrm{d}t \Big)\mathds{1}_{X_i \notin A}.
\end{align*}
As
\begin{align*}
    \frac{1}{n} \sum_{i = 1}^n \int \int |\tilde{v}_t(x)- v_t(x|X_i)|^2 p_t(x|X_i)\;\mathrm{d}x \; \mathrm{d}t &= \frac{1}{n} \sum_{i = 1}^n \Big(\int \int |\tilde{v}_t(x)- v_t(x|X_i)|^2 p_t(x|X_i)\;\mathrm{d}x \; \mathrm{d}t\Big)\mathds{1}_{X_i \in A} \\
    &\quad  + \frac{1}{n} \sum_{i = 1}^n \Big( \int \int |\tilde{v}_t(x)- v_t(x|X_i)|^2 p_t(x|X_i)\;\mathrm{d}x \; \mathrm{d}t\Big)\mathds{1}_{X_i \notin A} ,
\end{align*} we can use the other case of  \citet[Lemma 15]{chen2023b} to conclude that for every $a \in (0,1], \delta_2 \in (0,\frac{1}{3})$ and $\tau \in \mathbb{R}_{>0}$
\begin{align*}
    \mathbb{P} \Big( \sup_{\bar{v} \in \mathcal{M}}\frac{1}{n} &\sum_{i =1}^n g(\bar{v}, X_i)\mathds{1}_{X_i \in A}  - (1+a)\mathbb{E}_{Y \sim \mathbb{P}^*}[g(\bar{v}, Y)\mathds{1}_{Y \in A}] \\ & \quad >\frac{(1+ 3/a) D\log(n)^6}{3n} \log \Big( \frac{\mathcal{N}(\tau, g(\mathcal{M}), \|\cdot \|_{\infty})}{\delta_2}\Big) + (1+a) \tau\Big)\leq \delta_2.
\end{align*}
We conclude that with a probability of $1- \delta_2$
 \begin{align*}
     \frac{1}{n} \sum_{i = 1}^n \Big(\int \int &|\tilde{v}_t(x)- v_t(x|X_i)|^2 p_t(x|X_i)\;\mathrm{d}x \; \mathrm{d}t\Big)\mathds{1}_{X_i \in A} \\&\leq (1+a)\mathbb{E}_{Y \sim \mathbb{P}^*}[g(\tilde{v}, Y)\mathds{1}_{Y \in A}]+ \frac{(1+ 3/a) D\log(n)^6}{3 n} \log \Big( \frac{\mathcal{N}(\tau, g(\mathcal{M}), \|\cdot \|_{\infty})}{\delta_2}\Big)  + (1+a) \tau.
 \end{align*}
Like before, we separate $a\mathbb{E}_{Y \sim \mathbb{P}^*}[g(\tilde{v}, Y)\mathds{1}_{Y \in A}]$.
Using \Cref{thm:equivalence_constant} again, we obtain
\begin{align*}
    \mathbb{E}_{Y \sim \mathbb{P}^*}[g(\tilde{v}, Y)\mathds{1}_{Y \in A}] &= \mathbb{E}_{\substack{t\sim \mathcal{U}[0,1], \\Y \sim p^*, \\X_t \sim p_t\left(\cdot | Y\right)}}\big[\left|\tilde{v}_t(X_t)-v_t\left(X_t| Y\right)\right|^2\big] - \int_{\mathbb{R}^d\setminus A} \int \int |\tilde{v}_t(x)-v_t(x|y)|^2 p_t(x|y)\; \mathrm{d}x \; \mathrm{d}t \;p^*(x)\;\mathrm{d}y.
\end{align*}
Further
\begin{equation*}
    \mathbb{E}_{\substack{t\sim \mathcal{U}[0,1], \\Y \sim p^*, \\X_t \sim p_t\left(\cdot | Y\right)}}\big[\left|\tilde{v}_t(X_t)-v_t\left(X_t| Y\right)\right|^2\big] =  \mathbb{E}_{\substack{t \sim \mathcal{U}[0,1]\\X_t \sim p_t}} \big[|v_t(X_t) - \tilde{v}_t(X_t) |^2\big] +C.
\end{equation*}
Note that $C$ cancels with the same constant in \eqref{eq:error_decomp_1}.
The terms arising from the restriction of the Bernstein-bound to the set $A$ can be bounded with high probability by the following result.
\begin{lemma}\label{thm:remaining_terms}$ $
Fix $C \in \mathbb{R}$. Let $\log(\sigma_{\min}^{-1}) \sim \log(n)$. With a probability of $1- \frac{1}{3n}$ for $n \geq e^{3 (d+7+\frac{1}{2})}$
        \begin{equation*}
            |\mathbb{E}[g(\hat{v}, Y)\mathds{1}_{Y \notin A}] - \mathbb{E}[  g(\tilde{v}, Y)\mathds{1}_{Y \notin A}] + \sum_{i = 1}^n  g(\tilde{v}, X_i)\mathds{1}_{X_i \notin A} - g(\hat{v}, X_i)\mathds{1}_{X_i \notin A}| \leq n^{- \frac{1}{2}}.
        \end{equation*}
\end{lemma}

Collecting all other terms, inserting $\log(\sigma_{\min}^{-1})\sim \log(n)$, setting $\delta_1 =  \delta_2 = \frac{1}{3n}$ and using the union bound leads to the result. The infimum can be used due to the presence of other nonzero terms.
\end{proof}

\textit{Preparation for the proof of \Cref{thm:rate_smooth}.}\\
The proof of \Cref{thm:rate_smooth} requires some additional results to bound the covering number and apply the approximation result. We further need to look carefully into the approximation result by \cite{Guehring2020}.\\
As we cannot approximate a function on $\mathbb{R}^d$ with a finite neural network with fixed precision, we are going to look at functions that approximate $v$ on the set $[-\log(n), \log(n)]^d\times[0,1]$ and show that the error on the complement is small. 
Hence we want to determine a network necessary to obtain
 \begin{equation*}
 \int \int_{[-\log(n), \log(n)]^d} | \tilde{v}_t(x) - v_t(x) |^2 p_t(x) \; \mathrm{d}x \; \mathrm{d}t\leq \varepsilon,
\end{equation*} for a given approximation error $\varepsilon >0$.
We map any point outside of $[-\log(n), \log(n)]^d$ back to this hypercube using the one layer ReLU net associated with the following function
 \begin{equation*}
\operatorname{clip}_{\text{input}}(x_i, -\log(n), \log(n)) = x_i - \phi(x_i-\log(n)) + \phi(-x_i-\log(n)).
\end{equation*}
Thus, the Lipschitz constant of a network that approximates $v$ on the set $[-\log(n), \log(n)]^d\times[0,1]$ is bounded by the Lipschitz constant of the network on $[-\log(n), \log(n)]^d\times[0,1]$.
A similar one layer clipping function can be used to clip the component functions at \eqref{eq:bound_functions_M}. The order of the number of layers and the number of nonzero weights will not change with both adaptions.
For the error on the complement, we use the following result:
\begin{lemma}
    \label{thm:fast_decay}
   For every $\tilde{v}$ in $\mathcal{M}$ and $\log(\sigma_{\min}^{-1}) \sim \log(n)$, if $n \geq e^{4(\frac{1}{2}+6+d)}$ then
        \begin{equation*}
            \int_0^1 \int_{\mathbb{R}^d \setminus[-\log(n), \log(n)]^d}|v_t(x) - \tilde{v}_t(x) |^2 p_t(x) \; \mathrm{d}x \; \mathrm{d}t \leq n^{- \frac{1}{2}}.
        \end{equation*}
\end{lemma}

\underline{Bound of the covering number:}
\begin{lemma}
    \label{thm:Lipschitz_constant_g}
    We have for $\bar{v}^1, \bar{v}^2 \in \mathcal{M}$ and every $y \in [-\log(n), \log(n)]^d$, that 
    \begin{equation*}
    |g(\bar{v}^1, y) -  g(\bar{v}^2, y)| \lesssim\| \bar{v}^1- \bar{v}^2 \|_{\infty} \log(n)^4.
    \end{equation*}
\end{lemma}
Hence we can bound
$
\mathcal{N}(\tau, g(\mathcal{M}), \|\cdot \|_{\infty}) \leq \mathcal{N}( \frac{\tau}{\log(n)^4}, \mathcal{M}, \|\cdot \|_{\infty}).
$
From \citet[Lemma 3]{suzuki2018} we know that the covering number of the set ReLU networks of a $[0,1]^d$, denoted by $\mathsf{NN}_{\text{ReLU}}$ cube is
\begin{equation*}
\log(\mathcal{N}(\tau, \mathsf{NN}_{\text{ReLU}}, \|\cdot \|_{\infty, [0,1]^d}) )\leq 2 S L \log \left(\tau^{-1} L(B \vee 1)(W+1)\right) ,
\end{equation*} where $L$ is the number of layers, $S$ is the number of nonzero weights, $B$ is the maximal absolute value of a single weight and $W $ is the maximal width. A careful inspection of the proof of the approximation result by \cite{Guehring2020} will later reveal the specific choices. \newline
To consider functions on $[-\log(n),\log(n)]^d\times [0,1]$, we proceed analogously to \citet[p.46]{yakovlev2025}, multiplying the first $d$ coordinates of the weight matrix of the first lemma with $\log(n)$ and dividing the input vector with the same value. This has of course an impact on the bound of the weights, which will scale up by the factor $\log(n)$. Additionally, the weights of the clipping layers are of order $\log(n)^3$. We thus obtain
\begin{equation*}
\log(\mathcal{N}(\tau, g(\mathcal{M}), \|\cdot \|_{\infty})) \lesssim  2 S L \log \left(\tau^{-1}\log(n)^4 L(B\vee1)\log(n)^3 (W+1)\right). 
\end{equation*}

\underline{Approximation result:}\\
From \citet[Corollary 4.2]{Guehring2020} we know that for any $f \in C^{s}((0,1)^d), s \in \mathbb{N} \setminus \{ 1\}$ with bounded $s$-th order partial derivatives we can approximate $f$ and its Lipschitz constant simultaneously using a ReLU network as defined in \eqref{def:NN}. 
More precisely, for every $\varepsilon \in (0, 1/2)$ there exists a ReLU network $f_{\mathsf{NN}}$ such that
\begin{equation}\label{eq:approx_guehring}
\left\|f_{\mathsf{NN}}-f\right\|_{\mathcal{H}^1\left((0,1)^d\right)} \leq \varepsilon.
\end{equation}
The network $f_{\mathsf{NN}}$ can be chosen with at most
$
L \leq c \cdot \log \left(\varepsilon^{-\frac{s}{s-1}}\right)
$
hidden layers and 
$
S \leq c \cdot \varepsilon^{-\frac{d}{s-1}} \cdot \log ^2\left(\varepsilon^{-\frac{s}{s-1}}\right)
$
nonzero weights, where $c$ is a constant.

We remark that \cite{Guehring2020} use the term "standard neural networks" for these ReLU architectures, in contrast to networks with skip connections that they also study. Additionally, note that we need to use $d$ parallel networks in order to approximate the function $v$.\\
As the upper bound depends on the Lipschitz constant of the network, this stronger approximation result is necessary. Since the covering number depends on the maximum bound of the weights, we need to track this in the proof of \citet[Theorem 4.1]{Guehring2020}. \citet[Lemma C.3 (v)]{Guehring2020} reveals that the biggest single weight is set to $N$, which is later chosen as $\Big \lceil \Big( \frac{\varepsilon}{2 C L}\Big)^{-\frac{1}{s- 1}} \Big\rceil$, where $L$ is the bound on the partial derivatives up to the highest order considered. To apply the result of \citet[Corollary 4.2]{Guehring2020} without tracking the bound on the partial derivatives in their proof, we approximate $v$ divided by this bound and scale the approximated function up in the end. This can be achieved adding a layer in front and a layer after the network. The maximal weight used is of the size of the bound.  \newline
As $v$ is in $C^{\infty}$, we can apply \citet[Corollary 4.2]{Guehring2020} for arbitrary large $s$. However, the use of a larger $s$ will lead to a larger bound on the absolute value of the partial derivatives up to the order $s$. Since this scales up the approximation error, we need a precise quantification. The next result uses the $\log$-Sobolev inequality to bound higher derivatives of $v$.

\begin{lemma}\label{thm:higher_orders_bound} Fix $s \in \mathbb{N}$. Let $\{i_1,...,i_s\} \subset \{ 1,...,d\}^s$.
    For the $k$-th order partial derivative of $v$ we have that 
    \begin{equation*}
        \frac{\partial^s}{\partial x_{i_1},..., \partial x_{i_s}}v^j_t(x) \lesssim \log(\sigma_{\min}^{-1}) \sigma_{\min}^{-s+1}.
    \end{equation*}
\end{lemma}
For the derivatives with respect to $t$, we obtain
\begin{lemma}\label{thm:higher_order_t_bound}
    Fix $s \in \mathbb{N}$. Assume that $\log(\sigma_{\min})\geq s$. For every $k \in \{1,...,s \}$ we have that
        \begin{equation*}
        \frac{\partial^k}{\partial t^k}v^j_t(x) \lesssim \operatorname{polylog}(\sigma_{\min}^{-1}) \operatorname{polylog}(n)\sigma_{\min}^{-s-2}.
    \end{equation*}
\end{lemma}

The bound on mixed derivatives follows exactly the same lines as the proofs of \Cref{thm:higher_orders_bound} and \Cref{thm:higher_order_t_bound}

Now we are ready to prove \Cref{thm:rate_smooth}.

\begin{proof}[Proof of  \Cref{thm:rate_smooth}]
We scale the function down by $\sigma_{\min}^{-s-2} \log^2(n)$ and approximate $v \cdot\sigma_{\min}^{-s-2} \log^2(n) $ instead of $v$. After that, we add another layer to scale the output of the neural net up. The approximation error $\varepsilon$ will also scale up by $\sigma_{\min}^{-s-2} \log^2(n)$.
The maximal weight in the inner network will then no longer depend on this bound, but the largest weight might increase. Hence we can set
 $
W = C (\varepsilon^{-\frac{1}{s- 2}} \vee \sigma_{\min}^{-s-2} \log^2(n)).
$
Inserting everything into \Cref{thm:error_decomp} and setting $\tilde{d} = d+1$, we obtain for $a<1$ with a probability of $1-\frac{1}{n}$ for $n$ big enough and $I = [-\log(n), \log(n)]^d$
\begin{align}
&\mathbb{E}_{\substack{t \sim \mathcal{U}[0,1]\\X_t \sim p_t}} \big[|v_t(X_t) - \tilde{v}_t(X_t) |^2\big] \notag \lesssim  \int_0^1 \int_{\mathbb{R}^d \setminus I}|v_t(x) - \tilde{v}_t(x) |^2 p_t(x) \; \mathrm{d}x \; \mathrm{d}t  + (2+a) \tau + a (\log(n)^4 + \log(n)^6)\\ & + \varepsilon\sigma_{\min}^{-s-2} \log^2(n)  +\frac{ a^{-1} \log(n)^7}{n}  \varepsilon^{-\frac{\tilde{d}}{s-1}} \log\big(2 
 \big(\tau^{-1}\log(n)^7 L(B \vee 1)  (\varepsilon^{-\frac{1}{s- 1}} \vee \sigma_{\min}^{-s-2} \log^2(n)\big) \big)
.  \label{eq:error_decomp_plugin}
\end{align}
Note that $B$ and $L$ will depend on $\varepsilon$, thus we are restricting feasible choices of $\varepsilon$ to choices that grow at most polynomial in $n^{-1}$.
Ignoring logarithmic terms, we first solve for $a$.
$
    \frac{a^{-1}\varepsilon^{-\frac{\tilde{d}}{s-1}}}{n} \sim a $ is equivalent to $ a = \Big(\frac{\varepsilon^{-\frac{\tilde{d}}{s-1}}}{n} \Big)^{\frac{1}{2}}.
$
$\tau$ can be set such that the corresponding term is of the same order, it suffices to set 
$
\tau \sim \Big(\frac{\varepsilon^{-\frac{\tilde{d}}{s-1}}}{n} \Big)^{\frac{1}{2}} \Big(2+ \Big(\frac{\varepsilon^{-\frac{\tilde{d}}{s-1}}}{n} \Big)^{\frac{1}{2}}\Big)^{-1}.
$ This choice, as well as the choice of $\varepsilon$, influence the first term in \eqref{eq:error_decomp_plugin} only logarithmically. 
Now we solve for $\varepsilon$.
$
\Big(\frac{\varepsilon^{-\frac{\tilde{d}}{s-1}}}{n} \Big)^{\frac{1}{2}} \sim  \varepsilon\sigma_{\min}^{-s-2} $ is equivalent to $  \varepsilon \sim n^{- \frac{s-1}{\tilde{d}+2s-2}}\sigma_{\min}^{(s+2)\frac{2s-2}{\tilde{d}+2s-2}}.
$ Restricting $\sigma_{\min}$ to polynomials in $n$, we ensure $\varepsilon$ is polynomial in $n$.  Thus, ignoring logarithmic factors in $n$, we obtain with a probability of $1- \frac{1}{n} $ for $n$ big enough
\begin{equation}
    \mathbb{E}_{\substack{t \sim \mathcal{U}[0,1]\\X_t \sim p_t}} \big[|v_t(X_t) - \tilde{v}_t(X_t) |^2\big] \lesssim \operatorname{polylog}(n) n^{- \frac{s-1}{\tilde{d}+2s-2}}\sigma_{\min}^{(s+2)\big(\frac{2s-2}{\tilde{d}+2s-2}-1\big)} \notag.
\end{equation}

To choose $\sigma_{\min}$ such that is is optimal in the setting of \Cref{thm:error_decomp_easy}, we have to set it such that
$
\sigma_{\min}^{1+ \alpha} \sim \Big(n^{- \frac{s-1}{d+2s-2}}\sigma_{\min}^{(s+2)\big(\frac{2s-2}{\tilde{d}+2s-2}-1\big)} \Big)^{\frac{1}{2}}$ which is equivalent to $ \sigma_{\min} \sim n^{- \frac{1}{\frac{s+2}{s-1}\tilde{d}+ (2 \alpha + 2)(\frac{\tilde{d}}{s-1}+2 )}}.
$ This choice is a polynomial of $n^{-1}$ and hence a feasible choice for $\sigma_{\min}$.
For $n\geq 1$ and $\sigma_{\min}<1$, we know that the network approximates the Lipschitz constant of $v$ with precision $\varepsilon \sigma_{\min}^{-s-2}$. Plugging in the choice of $\sigma_{\min}$ leads to 
$
\varepsilon \sigma_{\min}^{-s-2} \sim n^{- \frac{2+2\alpha}{\frac{s+2}{s-1}\tilde{d}+ (2 \alpha + 2)(\frac{\tilde{d}}{s-1}+2 )}} \leq 1.
$
Hence we obtain for the Lipschitz constant $\hat{L}_t$ of $\hat{v}_t$ $
    e^{\int_0^1 \hat{L}_t\; \mathrm{d}t} \leq  e^{\int_0^1 \Gamma_t + 1\; \mathrm{d}t}   \leq e^{C + 1}.
$
Note that we cannot let $s \rightarrow \infty$, as this will blow up the constant. 
However, for every fixed $\eta>0$, we can choose $s$ such that
$
s =\big\lceil (5+2\alpha)\frac{\tilde{d}}{\eta} +1\big\rceil.
$
This leads to the final choice
$
\sigma_{\min} \sim n^{- \frac{1}{\tilde{d}+ 4\alpha + 4 + \eta}}
$
Now we use the results of \Cref{thm:remaining_terms}, set $\delta = \frac{1}{2n}$ and use the union bound. Combining this with \Cref{thm:error_decomp_easy}, \Cref{thm:examples_sigma_t}, \cite[Theorem 3.3]{kunkel2025minimax} and $\tilde{d} = d+1$, we obtain with a probability of $1 - \frac{1}{n} $
\begin{equation*}
   \mathsf{W}_1(\mathbb{P}^*, \mathbb{P}^{\hat{\psi}_1(Z)}) \lesssim \operatorname{polylog}(n) n^{- \frac{1+\alpha}{d+ 4\alpha + 5 + \eta}}.
\end{equation*} 
The number of hidden layers $L$ of $f_{\mathsf{NN}}$ is bounded by $ c \cdot \log \left(n\right)$ and the number of nonzero weights $S$ is bounded by $ c \cdot n^{c(d, \alpha, \eta)}\cdot \log^2\left(n\right)$, where $c$ is a constant independent from $n$.
\end{proof}

	\bibliography{literatur_2}
	
	\appendix
	\section{Additional proofs of \Cref{sec:proofs} and helper results}\label{sec:proofs_helper_lemma}
	\begin{lemma}\label{thm:equivalence_constant}
		Let $p_t(x) > 0$ for all $x \in \mathbb{R}^d$. In the above setting, it holds that for every measurable function $\tilde{v}$ and $v_t$ from \eqref{marginal_prob_paths}
		\begin{equation*}
			\mathbb{E}_{\substack{t \sim \mathcal{U}[0,1]\\X_t \sim p_t}} \big[|v_t(X_t) - \tilde{v}_t(X_t) |^2\big] = \mathbb{E}_{\substack{t\sim \mathcal{U}[0,1], \\Y \sim p^*, \\X_t \sim p_t\left(\cdot | Y\right)}}\big[\left|\tilde{v}_t(X_t)-v_t\left(X_t| Y\right)\right|^2\big] - \mathbb{E}_{\substack{t\sim \mathcal{U}[0,1], \\Y \sim p^*, \\X_t \sim p_t\left(\cdot | Y\right)}}[|v_t(X_t)-v_t(X_t|Y)|^2].
		\end{equation*}
	\end{lemma}
	\begin{proof}[Proof of \Cref{thm:equivalence_constant}]
		From \cite{lipman2023}, we know that there is a constant $C \in \mathbb{R}$ independent of $\tilde{v}$ such that 
		\begin{equation*}
			\mathbb{E}_{\substack{t \sim \mathcal{U}[0,1]\\X_t \sim p_t}} \big[|v_t(X_t) - \tilde{v}_t(X_t) |^2\big] = \mathbb{E}_{\substack{t\sim \mathcal{U}[0,1], \\Y \sim p^*, \\X_t \sim p_t\left(\cdot | Y\right)}}\big[\left|\tilde{v}_t(X_t)-v_t\left(X_t| Y\right)\right|^2\big] + C.
		\end{equation*}
		Setting $\tilde{v}_t = v_t$, we obtain
		\begin{equation*}
			0 = \mathbb{E}_{\substack{t\sim \mathcal{U}[0,1], \\Y \sim p^*, \\X_t \sim p_t\left(\cdot | Y\right)}}\big[\left|v_t(X_t)-v_t\left(X_t| Y\right)\right|^2\big] + C \quad \Longleftrightarrow \quad C = -  \mathbb{E}_{\substack{t\sim \mathcal{U}[0,1], \\Y \sim p^*, \\X_t \sim p_t\left(\cdot | Y\right)}}\big[\left|v_t(X_t)-v_t\left(X_t| Y\right)\right|^2\big]. \qedhere
		\end{equation*}
	\end{proof}

	\begin{proof}[Proof of \Cref{thm:jacobian}]

		By the definition of $v_t$ and $v_t(\cdot|\cdot)$, we get 
		\begin{equation*}
			v_t(x) = \int \big(\frac{\sigma_t^{\prime}}{\sigma_t} (x - \mu_t(y)) + \mu_t^{\prime}(y) \Big) \frac{p_t(x |y)}{p_t(x)}p^*(y)\; \mathrm{d}y,
		\end{equation*} where $^\prime$ indicates the derivative with respect to $t.$ For the Jacobian, we get using $\mu_t(y) = t^{\gamma}y$
		\begin{align}
			& D_x v_t(x)  = D_x \int \Big(\frac{\sigma_t^{\prime}}{\sigma_t} (x - \mu_t(y)) + \mu_t^{\prime} (y)\Big) \frac{p_t(x |y)}{p_t(x)}p^*(y)\; \mathrm{d}y \\ 
			& = D_x \frac{\sigma_t^{\prime}}{\sigma_t} \frac{x}{p_t(x)} \int p_t(x |y)p^*(y)\; \mathrm{d}y - D_x \frac{\sigma_t^{\prime}}{\sigma_t} \frac{t^{\gamma}}{p_t(x)} \int y p_t(x |y)p^*(y)\; \mathrm{d}y + D_x  \frac{\gamma t^{\gamma -1}}{p_t(x)} \int y p_t(x |y)p^*(y)\; \mathrm{d}y \\
			& = \frac{\sigma_t^{\prime}}{\sigma_t} I_d + \Big(\gamma t^{\gamma-1}- \frac{\sigma_t^{\prime}t^{\gamma}}{\sigma_t} \Big) D_x \frac{\int y p_t(x |y)p^*(y)\; \mathrm{d}y}{p_t(x)}.
		\end{align}
		For the derivative with respect to $x_i$ of the $j-th$ coordinate function, we get using the dominated convergence theorem
		\begin{align*}
			&\frac{\partial }{\partial x_i} \frac{\int y_j p_t(x|y)p^*(y)\; \mathrm{d}y}{\int p_t(x|y)p^*(y)\; \mathrm{d}y} \\& = \frac{( \int y_j  \frac{\partial }{\partial x_i} p_t(x|y)p^*(y)\; \mathrm{d}y)(\int p_t(x|y)p^*(y)\; \mathrm{d}y) -  ( \int y_j p_t(x|y)p^*(y)\; \mathrm{d}y)(\int \frac{\partial }{\partial x_i}p_t(x|y)p^*(y)\; \mathrm{d}y) }{(\int p_t(x|y)p^*(y)\; \mathrm{d}y)^2}\\
			& = \frac{t^{\gamma} \int y_j y_i p_t(x|y) p^*(y) \; \mathrm{d}y}{\sigma_t^2\int p_t(x|y)p^*(y)\; \mathrm{d}y} - \frac{t^{\gamma}(\int y_j p_t(x|y)p^*(y)\; \mathrm{d}y)(\int y_i p_t(x|y)p^*(y)\; \mathrm{d}y)}{\sigma_t^2(\int p_t(x|y)p^*(y)\; \mathrm{d}y)^2}.
		\end{align*}
		For fixed $x$ and $t$ let $Y^{x,t}$ be a random variable with density $\frac{p_t(x|\cdot)p^*(\cdot )}{\int p_t(x|y)p^*(y)\; \mathrm{d}y}.$ Then for $t \in [0,1)$
		\begin{equation*}
			\frac{\partial }{\partial x_i} \frac{\int y_j p_t(x|y)p^*(y)\; \mathrm{d}y}{\int p_t(x|y)p^*(y)\; \mathrm{d}y}  = \frac{t^{\gamma}}{\sigma_t^2} (\mathbb{E}[Y_i^{x,t} Y_j^{x,t}] - \mathbb{E}[Y_i^{x,t}]\mathbb{E}[ Y_j^{x,t}]) = \frac{t^{\gamma}}{\sigma_t^2} \operatorname{Cov}(Y^{x,t})_{ji}.
		\end{equation*}
		We obtain for the derivative of with respect to $x_i$ of the $j-th$ coordinate function
		
		\begin{equation*}
			\frac{\partial}{\partial x_i} v^j_t(x) = \frac{\sigma_t^{\prime}}{\sigma_t} \mathds{1}_{i = j} + \Big(\gamma t^{\gamma-1}- \frac{\sigma_t^{\prime}t^{\gamma}}{\sigma_t} \Big)  \frac{t^{\gamma}}{\sigma_t^2} \operatorname{Cov}(Y^{x,t})_{ji}.\qedhere
		\end{equation*}  
	\end{proof}
	
	\begin{proof}[Proof of \Cref{thm:easy_lipschitz}]

		Let $t^*$ be such that $\sigma_{t^*} = \frac{1}{\vartheta}$. Then 
		\begin{equation*}
			\begin{aligned}
				\int_0^{t^{*}}\Big|  \frac{\sigma_t^{\prime}}{\sigma_t} \mathds{1}_{i = j} + \Big(\gamma t^{\gamma -1}&- \frac{\sigma_t^{\prime}t^{\gamma}}{\sigma_t} \Big)  \frac{t^{\gamma}}{\sigma_t^2} \operatorname{Cov}(Y^{\cdot,t})_{ij}    \Big|\; \mathrm{d}t  \int_0^{t^{*}}\Big|  \frac{\sigma_t^{\prime}}{\sigma_t}\Big|\; \mathrm{d}t +  \int_0^{t^{*}}\Big| \Big(\gamma t^{\gamma -1}- \frac{\sigma_t^{\prime}t^{\gamma}}{\sigma_t} \Big)  \frac{t^{\gamma}}{\sigma_t^2} \operatorname{Cov}(Y^{\cdot,t})_{ij}    \Big|\; \mathrm{d}t. 
			\end{aligned}    
		\end{equation*}
		For the first term, we can use that $\sigma_t^{\prime}$ is non positive and hence by change of variables
		\begin{equation*}
			\int_0^{t^{*}}\Big|  \frac{\sigma_t^{\prime}}{\sigma_t}\Big|\; \mathrm{d}t = -  \int_0^{t^{*}}  \frac{\sigma_t^{\prime}}{\sigma_t}\; \mathrm{d}t = \int_{\sigma_{t^*}}^{1} \frac{1}{u}\; \mathrm{d}u = \log(\sigma_{t^*}^{-1}) = \log(\vartheta).
		\end{equation*}
		For the second integral, we can bound the Covariance term by
		$
		\big|\operatorname{Cov}(Y^{\cdot,t})_{ij}\big|   \leq C.
		$  Now we get 
		\begin{align*}
			\int_0^{t^{*}} &\Big| \Big(\gamma t^{\gamma -1}- \frac{\sigma_t^{\prime}t^{\gamma}}{\sigma_t} \Big)  \frac{t^{\gamma}}{\sigma_t^2} \operatorname{Cov}(Y^{\cdot,t})_{ij}    \Big|\; \mathrm{d}t  \leq C  \int_0^{t^{*}}\Big| \Big(\gamma t^{\gamma -1}- \frac{\sigma_t^{\prime}t^{\gamma}}{\sigma_t} \Big)  \frac{t^{\gamma}}{\sigma_t^2}  \Big|\; \mathrm{d}t\\
			& \leq  C  \int_0^{t^{*}}\Big|\gamma t^{\gamma -1}\frac{t^{\gamma}}{\sigma_t^2}\Big|\; \mathrm{d}t + C \int_0^{t^{*}}\Big|\frac{\sigma_t^{\prime}t^{2\gamma}}{\sigma^3_t}\Big|    \; \mathrm{d}t \leq \vartheta^2 C \int_0^{t^{*}}\gamma t^{\gamma -1}\; \mathrm{d}t + \vartheta^2 C\int_0^{t^{*}} \Big|  \frac{\sigma_t^{\prime}}{\sigma_t}\Big|   \; \mathrm{d}t \lesssim \vartheta^2 (1 + \log(\vartheta)).\qedhere
		\end{align*}
	\end{proof}
	
	\begin{proof}[Proof of \Cref{thm:remaining_terms}]
		First we abbreviate
		\begin{equation*}
			B \coloneqq \mathbb{E}[g(\hat{v}, Y)\mathds{1}_{Y \notin A}] - \mathbb{E}[  g(\tilde{v}, Y)\mathds{1}_{Y \notin A}] + \sum_{i = 1}^n  g(\tilde{v}, X_i)\mathds{1}_{X_i \notin A} - g(\hat{v}, X_i)\mathds{1}_{X_i \notin A}.
		\end{equation*}
		Both $|\tilde{v}|_{\infty}$ and $|\hat{v}|_{\infty}$ are bounded by $D^{\prime}\log(n)^3$, where $D^{\prime}$ is a constant collecting all terms in \eqref{eq:bound_functions_M}. In the proof of \Cref{thm:error_decomp}, we showed that
$
			g(v,y)  \lesssim \log(n)^6+ \log(n)^2|y|^2 + \log(\sigma_{\min}^{-1})^2 \log(n)^2 |y|^2.
	$  Thus, we obtain 
		\begin{equation*}
			\mathbb{E}_{X_i}[|B|] \lesssim \mathbb{E}\big[\big| \mathds{1}_{Y \notin A} \big(\log(n)^6 +  \log(n)^2 |Y|^2\big)\big|\big] = \log(n)^6  \int_{\mathbb{R}^d \setminus A} p^*(y) \; \mathrm{d} y +  \log(n)^2  \int_{\mathbb{R}^d \setminus A} |y|^2p^*(y) \; \mathrm{d} y.
		\end{equation*}
		Using the structure of $p^*$ and assuming $n \geq 4$, we obtain
		\begin{equation*}
			\int_{\mathbb{R}^d \setminus A} p^*(y) \; \mathrm{d} y   \leq   \frac{\int_{\mathbb{R}^d \setminus A}|y|^2\exp(- \frac{|y|^2}{2} - a(y))\; \mathrm{d}y }{\int\exp(- \frac{|y|^2}{2} - a(y))\; \mathrm{d}y}\leq e^{2L}\Big(2\frac{\int_{\log(n)}^{\infty} y_1^2 \exp(- \frac{y_1^2}{2} )\; \mathrm{d}y_1 }{\int\exp(- \frac{y_1^2}{2} )\; \mathrm{d}y_1} \Big)^d. 
		\end{equation*}
		As
		\begin{equation*}
			0 \leq \int_{\log(n)}^{\infty} y_1^2 \exp\Big(- \frac{y_1^2}{2} \Big)\; \mathrm{d}y = \Big[-y_1 \exp\Big(- \frac{y_1^2}{2}\Big)\Big]_{\log(n)}^{\infty}- \int_{\log(n)}^{\infty}\exp\Big(- \frac{y_1^2}{2}\Big)\; \mathrm{d}y_1\leq \log(n)\exp\Big(- \frac{\log(n)^2}{2}\Big),
		\end{equation*} we can bound the above by
		\begin{equation*}
			e^{2L}\Big(2\frac{\int_{\log(n)}^{\infty} y_1^2 \exp(- \frac{y_1^2}{2} )\; \mathrm{d}y_1 }{\int\exp(- \frac{y_1^2}{2} )\; \mathrm{d}y_1} \Big)^d \lesssim 
			\Big(\log(n)\exp\Big(- \frac{\log(n)^2}{2}\Big) \Big)^d. 
		\end{equation*}
		For large $n$, this decays faster than any polynomial in $n$. We conclude
		\begin{equation*}
			\mathbb{E}_{X_i}[|B|] \lesssim \log(n)^6\Big(\log(n)\exp\Big(- \frac{\log(n)^2}{2}\Big) \Big)^d.
		\end{equation*}
		Now Markovs inequality yields 
		\begin{equation*}
			\mathbb{P}(B \geq n^{-\frac{1}{2}})  \lesssim   \log(n)^6\Big(\log(n)\exp\Big(- \frac{\log(n)^2}{2}\Big) \Big)^d n^{\frac{1}{2}},\quad 
			\mathbb{P}(B \leq - n^{-\frac{1}{2}})  \lesssim   \log(n)^6\exp\Big(- \frac{\log(n)^2}{2}\Big) \Big)^d n^{\frac{1}{2}}.
		\end{equation*} Notice that we can shift the constant as a factor of the bound of $B$. As $\log(\sigma_{\min}^{-1}) \sim \log(n)$,
		we need to choose $n$ large enough such that
		\begin{equation*}
			\log(n)^6\Big(\log(n)\exp\Big(- \frac{\log(n)^2}{2}\Big) \Big)^d n^{\frac{1}{2}} \lesssim \frac{1}{3n} \quad \Longleftarrow \quad n \gtrsim e^{3(d+7+\frac{1}{2})},
		\end{equation*}
		where the implication stems from a loose upper bound using $\log(n)^{6+d} \leq n^{6+d}$. This finishes the proof.
	\end{proof}
	
	\begin{proof}[Proof of \Cref{thm:fast_decay}]
		From the proof of \Cref{thm:error_decomp} and as $\log(\sigma_{\min}^{-1}) \sim \log(n)$ we now that
		\begin{equation*}
			|v_t(x)| \lesssim \log(n)^2|x|+\log(n), \quad 
			|\tilde{v}_t(x)| \lesssim \log(n)^3.
		\end{equation*}
		Thus, setting $I = [-\log(n), \log(n)]^d$
		\begin{align*}
			\int_0^1 \int_{\mathbb{R}^d \setminus I}|v_t(x) - \tilde{v}_t(x) |^2 p_t(x) \; \mathrm{d}x \; \mathrm{d}t& \lesssim  \log(n)^6 \int_0^1 \int_{\mathbb{R}^d \setminus I}p_t(x) \; \mathrm{d}x +  \log(n)^4 \int_0^1 \int_{\mathbb{R}^d \setminus I} |x|^2 p_t(x) \; \mathrm{d}x.
		\end{align*}
		By the definition of $p_t$ and due to the fact that the convolution of two densities is again a density, we know that
		\begin{equation*}
			p_t(x)  = \int p_t(x|y)p^*(y)\; \mathrm{dy} = \frac{\int \exp(- \frac{|x-ty|^2}{2 \sigma_t^2} - \frac{|y|^2}{2} -a(y))\; \mathrm{d}y}{\int \int \exp(- \frac{|x-ty|^2}{2 \sigma_t^2} - \frac{|y|^2}{2} -a(y))\; \mathrm{d}y \; \mathrm{d}x}e^{2L}\frac{\int \exp(- \frac{|x-ty|^2}{2 \sigma_t^2} - \frac{|y|^2}{2} )\; \mathrm{d}y}{\int \int \exp(- \frac{|x-ty|^2}{2 \sigma_t^2} - \frac{|y|^2}{2} )\; \mathrm{d}y \; \mathrm{d}x}.
		\end{equation*}
		Now for the inner integrals, we obtain
		\begin{equation*}
			\int \exp\Big(- \frac{|x-ty|^2}{2 \sigma_t^2} - \frac{|y|^2}{2} \Big)\; \mathrm{d}y  = (2 \pi)^{d / 2}\Big(\frac{\sigma_t^2}{t^2+\sigma_t^2}\Big)^{d / 2} \exp \Big(-\frac{|x|^2}{2\left(t^2+\sigma_t^2\right)}\Big)
			.
		\end{equation*}
		Inserting this into the bound and collecting all factors of $|x|^2$, we obtain
		\begin{equation*}
			p_t(x) \leq e^{2L} \frac{\exp(- \frac{1}{2}\frac{1}{\sigma_t^2 + t^2}|x|^2)}{\int \exp(- \frac{1}{2}\frac{1}{\sigma_t^2 + t^2} |x|^2)\; \mathrm{d}x} = \frac{e^{2L} }{(2 \pi (\sigma_t^2 + t^2))^{\frac{d}{2}}} \exp\Big(- \frac{1}{2}\frac{1}{\sigma_t^2 + t^2}|x|^2\Big).
		\end{equation*}
		Now we bound
		\begin{align*}
			\int_0^1 \int_{\mathbb{R}^d \setminus I}&p_t(x) \; \mathrm{d}x\; \mathrm{d}t \leq  \int_0^1 \int_{\mathbb{R}^d \setminus I} |x|^2 p_t(x) \; \mathrm{d}x\; \mathrm{d}t \leq  \int_0^1 \frac{e^{2L} }{(2 \pi (\sigma_t^2 + t^2))^{\frac{d}{2}}}\int_{\mathbb{R}^d \setminus I} |x|^2  \exp\Big(- \frac{1}{2}\frac{1}{\sigma_t^2 + t^2}|x|^2\Big) \; \mathrm{d}x\; \mathrm{d}t\\
			& =2 \int_0^1 \frac{e^{2L} }{(2 \pi (\sigma_t^2 + t^2))^{\frac{d}{2}}}\Big(\int_{\log(n)}^{\infty} x_1^2  \exp\Big(- \frac{1}{2}\frac{1}{\sigma_t^2 + t^2}x_1^2\Big) \; \mathrm{d}x_1\Big)^d\; \mathrm{d}t\\
			& \leq 2 \int_0^1 \frac{e^{2L} }{(2 \pi (\sigma_t^2 + t^2))^{\frac{d}{2}}}\Big(\log(n)(\sigma_t^2+t^2)\exp\Big(- \frac{1}{2}\frac{1}{\sigma_t^2+t^2}\log(n)^2\Big)\Big)^d\; \mathrm{d}t,
		\end{align*}
		where the last inequality follows from the same arguments as the bound in the proof of the first part of this Lemma. As $(\sigma_t^2+t^2) \leq 2$, we can bound
		\begin{equation*}
			\int_0^1 \int_{\mathbb{R}^d \setminus I}p_t(x) \; \mathrm{d}x\; \mathrm{d}t \lesssim \log(n)^d \exp\Big(- \frac{1}{4} \log^2(n)\Big).
		\end{equation*}
		Overall, we need to choose $n$ big enough such that
		$
		\log(n)^{6 + d}\exp(- \frac{1}{4} \log^2(n)) \leq n^{- \frac{1}{2}}$ which is implied by $ n \geq e^{4(\frac{1}{2}+6+d)},
		$
		where again the implication stems from a loose upper bound using $\log(n)^{6+d} \leq n^{6+d}$.
	\end{proof}

	\begin{proof}[Proof of \Cref{thm:Lipschitz_constant_g}]
		Similar to \cite[p.44]{yakovlev2025}, we bound
		\begin{align*}
			\Big|\int_0^1 \int &|\bar{v}_t^1(x) - v_t(x|y)(x)|^2 p_t(x|y)\; \mathrm{d}x\; \mathrm{d}t - \int_0^1 \int |\bar{v}_t^2(x) - v_t(x|y)(x)|^2 p_t(x|y)\; \mathrm{d}x\; \mathrm{d}t\Big| \\
			& \lesssim \int_0^1 \int |\bar{v}_t^1(x) - \bar{v}_t^2(x)| (|\bar{v}_t^1(x)|  + |\bar{v}_t^2(x)| + 2 |v_t(x|y)|  ) p_t(x|y)\; \mathrm{d}x\; \mathrm{d}t. 
		\end{align*}
		For $x \in \mathbb{R}^d$ and $y \in [-\log(n),\log(n)]^d$, we have 
		\begin{equation*}    |v_t(x|y)| = \Big|\frac{\sigma_t^{\prime}}{\sigma_t} (x+ (\sigma_{t}-t)y) \Big| \leq \Big|\frac{\sigma_t^{\prime}}{\sigma_t}  \Big| (|x|+|y|)\leq \Big|\frac{\sigma_t^{\prime}}{\sigma_t}\Big|  (|x| + \sqrt{d} \log(n)).
		\end{equation*}
		By \eqref{eq:bound_v_t} all $\bar{v} \in \mathcal{M}$ are such that $|\bar{v}_t(x)|\lesssim \log(n)^3$.  
		Therefore
		\begin{align*}
			\int_0^1 \int &|\bar{v}_t^1(x) - \bar{v}_t^2(x)| (|\bar{v}_t^1(x)|  + |\bar{v}_t^2(x)| + 2 |v_t(x|y)|  ) p_t(x|y)\; \mathrm{d}x\; \mathrm{d}t \\& \lesssim  \| \bar{v}^1- \bar{v}^2 \|_{\infty}(\log(n)^3 + \log(n) \int_0^1\Big|\frac{\sigma_t^{\prime}}{\sigma_t}  \Big| \int |x|p_t(x|y) \; \mathrm{d}x\; \mathrm{d}t) \lesssim \| \bar{v}^1- \bar{v}^2 \|_{\infty} \log(n)^4.
		\end{align*}
		The last inequality follows from Jensens inequality and the properties of the Gaussian.
	\end{proof}
	
	\begin{proof}[Proof of \Cref{thm:higher_orders_bound}]
		
		For an $s$-th order partial derivative and $\{i_1,...,i_s\} \subset \{ 1,...,d\}^s$, we know from the definition of $v_t$ 
		\begin{equation}
			\frac{\partial^s}{\partial x_{i_1},..., \partial x_{i_s}}v^j_t(x)  = \mathds{1}_{s = 1} \frac{\sigma_t^{\prime}}{\sigma_t} + (1- \frac{\sigma_t^{\prime}t}{\sigma_t}) \frac{\partial^s}{\partial x_{i_1},..., \partial x_{i_s}} \mathbb{E}_{q_{t,x}}[Y_j],
		\end{equation}
		where $
			q_{t,x} \propto \propto \exp(\langle \eta, Y\rangle)h(y),
		$
		with $\eta = \frac{t}{\sigma_t^2}x$ and $h(y) = \exp(- (\frac{t^2}{2\sigma_t^2}+ \frac{1}{2})\|y\|^2-a(y))$. Define
		$
		A(\eta) \coloneqq \log \int \exp(\langle \eta, Y\rangle)h(y)\; \mathrm{d}y.
		$
		In context of exponential families, $A$ is typically called the log-partition function.
		Then 
		$
		\frac{\partial}{\partial \eta_j} A(\eta) = \mathbb{E}_{q_{t,x}}[Y_j].
		$
		Thus, 
		\begin{equation*}
			\frac{\partial^s}{\partial x_{i_1},..., \partial x_{i_s}} \mathbb{E}_{q_{t,x}}[Y_j] =   \frac{\partial^s}{\partial x_{i_1},..., \partial x_{i_s}}  \frac{\partial}{\partial \eta_j} A(\eta) = \big(\frac{t}{\sigma_t^2} \big)^{s}\frac{\partial^s}{\partial \eta_{i_1},..., \partial \eta_{i_s}}  \frac{\partial}{\partial \eta_j} A(\eta). 
		\end{equation*}
		Further define the cumulant generating function as
		$
		K(\lambda) \coloneqq \log\big(\mathbb{E}[\exp(\langle \lambda, Y \rangle)] \big).
		$
		Then
		\begin{equation*}
			\mathbb{E}[\exp(\langle \lambda, Y)] = \exp(A(\lambda + \eta )- A(\eta)), \quad K(\lambda) = A(\lambda + \eta )- A(\eta).
		\end{equation*}
		Differentiating in $\lambda$ yields
		\begin{equation*}
			\frac{\partial^{s+1}}{\partial \lambda_{i_1},..., \partial \lambda_{i_s}, \partial \lambda_j} K(\lambda)\Big|_{\lambda = 0} =  \frac{\partial^{s+1}}{\partial \lambda_{i_1},..., \partial \lambda_{i_s}, \partial \lambda_j} A(\lambda + \eta )- A(\eta)\Big|_{\lambda = 0}= \frac{\partial^{s+1}}{\partial \eta_{i_1},..., \partial \eta_{i_s}, \partial \eta_j} A( \eta ).
		\end{equation*}
		Hence, 
		\begin{equation*}
			\frac{\partial^s}{\partial x_{i_1},..., \partial x_{i_s}}v^j_t(x)  = \mathds{1}_{s = 1} \frac{\sigma_t^{\prime}}{\sigma_t} + (1- \frac{\sigma_t^{\prime}t}{\sigma_t})  \big(\frac{t}{\sigma_t^2} \big)^{s} \kappa(Y_{i_1}, ..., Y_{i_s}, Y_j),
		\end{equation*}
		where $\kappa(Y_{i_1}, ..., Y_{i_s}, Y_j)$ is the joint cumulant, of $Y_{i_1}, ..., Y_{i_s}, Y_j$, which is defined as exactly this partial derivative. For $s \geq 1$, this cumulant is shift invariant, hence
		\begin{equation*}
			\kappa(Y_{i_1}, ..., Y_{i_s}, Y_j) = \kappa(Y_{i_1}- \mathbb{E}[Y_{i_1}], ..., Y_{i_s}- \mathbb{E}[Y_{i2}], Y_j- \mathbb{E}[Y_{j}]).
		\end{equation*}
		Using the formula from \cite{Leonov1959} we can express the cumulant using products of mixed moments
		\begin{align}
			|\kappa(Y_{i_1}- \mathbb{E}[Y_{i_1}], ..., Y_{i_s}- \mathbb{E}[Y_{i_s}], Y_j- \mathbb{E}[Y_{j}])| &= \sum_\pi(|\pi|-1)! \prod_{B \in \pi} \mathbb{E}\Big(\prod_{\ell \in B}| Y_{\ell} - \mathbb{E}[Y_{\ell}]|\Big) \notag \\
			& \leq \Big(\sum_\pi(|\pi|-1)! \Big) \prod_{\ell = 1}^{s+1} (\mathbb{E}[| Y_{B_{\ell}} - \mathbb{E}[Y_{B_{\ell}}]|^{s+1}])^{\frac{1}{s+1}},\label{eq:moment_bound}
		\end{align}
		where $\sum_{\pi}$ represents the sum over all partitions of $\{i_1,...,i_s,j\}$ and $B_{\ell}$ is the $\ell$-th element of $\{i_1,...,i_s,j\}$.

		To control the moments up to order $s+1$, we need a stronger argument than the Poincaré inequality used for the variance bound. Thus, we recall \Cref{def:log_sobolev}, the definition of the $\log$-Sobolev inequality.
	 Similar to the beginning of the proof of \Cref{thm:meh}, we conclude that the distribution with density \eqref{eq:density_gaussian} satisfies the $\log$-Sobolev inequality with $\log$-Sobolev constant $(1+ \frac{t^2}{\sigma_t^2})$.
		By \citet[Lemma 1.2]{Ledoux2001} we known that a bounded perturbation will again only impact the $\log$-Sobolev constant by an exponential term. Thus, we can bound the $\log$-Sobolev constant of the distribution of $Y^{x,t}$ by  
$
			\lambda \geq  e^{-4L}(1+ \frac{t^2}{\sigma_t^2}).
$
		
		Now we can use Herbst's argument (along the lines of \citet[p.148]{Ledoux1999}, note that our function is so simple that we can drop the assumption of a bounded function) to conclude that $Y^{x,t}_i - \mathbb{E}[Y^{x,t}_i]$ is subgaussian, i.e.\

		$
		\mathbb{P}(|Y_i - \mathbb{E}[Y_i]| \geq  r) \leq 2 e^{- \frac{r^2\lambda}{2}}.
		$
	Then we can bound the $s+1$-th moment along the lines of \citet[Proposition 2.5.2]{Vershynin2018} via 
		\begin{align*}
			\mathbb{E}[|Y_i - \mathbb{E}[Y_i]|^{s+1}] &= \int_0^{\infty} \mathbb{P}(|Y_i - \mathbb{E}[Y_i]|^{s+1} \geq u) \; \mathrm{d}u = \int_0^{\infty} \mathbb{P}(|Y_i - \mathbb{E}[Y_i]| \geq r) r^s (s+1)\; \mathrm{d}r\\
			& \leq 2 (s+1) \int_0^{\infty} e^{- \frac{r^2\lambda}{2}} r^{s} \; \mathrm{d}r = 2^{\frac{s+1}{2}} \lambda^{- \frac{s+1}{2}}\Gamma\Big(\frac{s+1}{2}\Big).
		\end{align*}
		Taking the $s+1$-th root leads to the bound
		$
		\mathbb{E}[|Y_i - \mathbb{E}[Y_i]|^{s+1}]^{\frac{1}{s+1}} \lesssim \lambda^{-\frac{1}{2}}.
		$
		Plugging in the $\log$ Sobolev constant of $Y$, we can control all moments via
		\begin{equation*}
			\mathbb{E}[|Y_{i_j}^{x,t}- \mathbb{E}[Y_{i_j}^{x,t}]|^{s+1}]^{\frac{1}{s+1}}  \lesssim \frac{\sigma_t}{\sqrt{\sigma_t^2 + t^2}} \leq \min\Big(1, \frac{\sigma_t}{t}\Big).
		\end{equation*}
		Thus we can bound \eqref{eq:moment_bound} by
		\begin{equation*}
			\prod_{\ell = 1}^{s+1} (\mathbb{E}[| Y_{B_{\ell}} - \mathbb{E}[Y_{B_{\ell}}]|^{s+1}])^{\frac{1}{s+1}}\lesssim \min\Big(1,\frac{\sigma_t^{s+1}}{t^{s+1}}\Big).
		\end{equation*}
		We obtain for 
		\begin{equation*}
			\frac{\partial^s}{\partial x_{i_1},..., \partial x_{i_s}}v^j_t(x)  \lesssim  \mathds{1}_{s = 1} \frac{\sigma_t^{\prime}}{\sigma_t} + \Big(1- \frac{\sigma_t^{\prime}t}{\sigma_t}\Big)\Big(\frac{t}{\sigma_t^2} \Big)^{s}\min\Big(1,\frac{\sigma_t^{s+1}}{t^{s+1}}\Big).
		\end{equation*}
		Minimizing in $\sigma_t$ and using $\frac{\sigma_t^{\prime}}{\sigma_t} = \log(\sigma_{\min})$ leads to the final bound.
	\end{proof}
	
	\begin{proof}[Proof of \Cref{thm:higher_order_t_bound}]
		Define
		\begin{equation*}
			q_{x,t} (y) \coloneqq \frac{ \exp(- \frac{|x-ty|^2}{2\sigma_t^2} - \frac{|y|^2}{2}- a(y))}{\int  \exp(- \frac{|x-ty|^2}{2\sigma_t^2} - \frac{|y|^2}{2}- a(y))\; \mathrm{d}y}, \quad w_{x,t}(y)\coloneqq \exp\Big(- \frac{|x-ty|^2}{2\sigma_t^2} - \frac{|y|^2}{2}- a(y)\Big).
		\end{equation*}
		For $k = 1$ we obtain using the dominated convergence theorem, the quotient rule and the derivative of the logarithm
		\begin{equation}\label{eq:derivative_t_1}
				\frac{\partial}{\partial t} v^j_t(x)  = \frac{\partial}{\partial t}  \mathbb{E}_{Y \sim q_{x,t}}[v_t(x|Y)_j] =  \mathbb{E}_{Y \sim q_{x,t}}\Big[\frac{\partial}{\partial t} v_t(x|Y)_j\Big] + \operatorname{Cov}\Big( v_t(x|Y)_j, \frac{\partial}{\partial t} \log( w_{x,t}(Y))\Big).
		\end{equation}
		We can use the same structure to obtain higher derivatives. Thus, all terms occurring in an derivative of $v_t$ of order $k$ are either derivatives of $v_t(x|y)_j$, derivatives of $\log( w_{x,t}(Y))$ (including the functions themselves as $0$-th order derivative) or products of these derivatives.  
		Since
		\begin{equation}\label{eq:derivative_t_5}	
				v_t(x|y)_j  =  \log(\sigma_{\min}) (x_j-ty_j)+y_j,\quad 
				\frac{\partial}{\partial t} v_t(x|y)_j  =  \log(\sigma_{\min})y_j,
		\end{equation}
		all higher derivatives of $v_t(\cdot|\cdot)$ vanish.
		For the second term occurring in \eqref{eq:derivative_t_1}, we have that
		\begin{align} 
			\frac{\partial}{\partial t} \log( w_{x,t}(Y)) =  - \frac{\partial}{\partial t}  \frac{1}{2}(|x-ty|^2 )\sigma_t^{-2}  =   ( \langle x,y\rangle - t|y|^2)\sigma_t^{-2} + \frac{1}{2}|x-ty|^2  \cdot 2  \log(\sigma_{\min}) \sigma_t^{-2}.\label{eq:derivative_t_4}
		\end{align}
		We conclude that
		\begin{equation}\label{eq:derivative_t_3}
			\frac{\partial^k}{\partial t^k}   \log( w_{x,t}(Y))  \lesssim  \operatorname{poly}(x,y,t) \operatorname{polylog}(\sigma_{\min}^{-1}) \sigma_t^{-2},
		\end{equation}
		where $\operatorname{poly}(x,y,t) $ is a again polynomial of finite degree in the components of $x,y,t$.
		Thus, bounds that are worse than $\sigma_t^{-2}$ can only appear when $ \frac{\partial}{\partial t} \log( w_{x,t}(Y))$ occurs raised to a power. Looking at \eqref{eq:derivative_t_1}, we see that the highest power of $ \frac{\partial}{\partial t} \log( w_{x,t}(Y))$ increases by $1$ each time the derivative is taken via the covariance term. Hence in the $k$-th derivative, the highest order terms are of the form 
		\begin{equation*}
			\operatorname{Cov}(f(Y)( \frac{\partial}{\partial t} \log( w_{x,t}(Y)))^{k-1}, \frac{\partial}{\partial t} \log( w_{x,t}(Y))),
		\end{equation*}
		where either $f \equiv 1$ or $f$ is of the form \eqref{eq:derivative_t_5} or \eqref{eq:derivative_t_3}. First, let $f \equiv 1$. Then due to \eqref{eq:derivative_t_4}
		\begin{align*}
		&	\operatorname{Cov}(( \frac{\partial}{\partial t} \log( w_{x,t}(Y)))^{k-1}, \frac{\partial}{\partial t} \log( w_{x,t}(Y)))= \mathbb{E}\Big[\Big(\frac{\partial}{\partial t} \log( w_{x,t}(Y))- \mathbb{E}\Big[\frac{\partial}{\partial t} \log( w_{x,t}(Y)) \Big] \Big)^k \Big]\\
			& = \sigma_t^{-2k} \mathbb{E}\Big[ \Big(\langle (1-t 2 \log(\sigma_{\min})) x ,Y\rangle - \mathbb{E}[\langle (1-t 2 \log(\sigma_{\min})) x ,Y\rangle]  \\ & \quad + (t^2 2 \log(\sigma_{\min}) - t)|Y|^2 - \mathbb{E}[(t^2 2  \log(\sigma_{\min}) - t)|Y^2|]\Big)^k \Big]\\
			& \lesssim \sigma_t^{-2k} \sum_{i = 1}^d (1-t 2  \log(\sigma_{\min}))^k x^k \mathbb{E}\Big[ \big( Y_i - \mathbb{E}[ Y_i]\big)^k \Big] +\sum_{i = 1}^d (t^2 2  \log(\sigma_{\min})t- t)^k\mathbb{E}\Big[\big( Y_i^2 - \mathbb{E}[Y_i^2]\big)^k \Big],
		\end{align*}
		where all expectations are taken with respect to $Y \sim q_{x,t}$. From the proof of \Cref{thm:higher_orders_bound} we know that 
		\begin{equation*}
			\mathbb{E}\Big[ \big( Y_i - \mathbb{E}[ Y_i]\big)^k \Big] \lesssim \Big( \frac{\sigma_t^2}{t^2+\sigma_t^2}\Big)^{\frac{k}{2}}  \leq  \min\Big( \frac{\sigma_t^2}{t^2}, 1\Big)^{\frac{k}{2}}
			= \min\Big( \frac{\sigma_t^k}{t^k}, 1\Big).
		\end{equation*}
		By \citet[Lemma 2.7.6]{Vershynin2018} combined with \citet[Lemma 2.7.1]{Vershynin2018} we know that if $Y_i - \mathbb{E}[Y_i]$ is subgaussian, which was shown in \Cref{thm:higher_orders_bound}, then $(Y_i - \mathbb{E}[Y_i])^2$ is subexponential with the same order of subgaussian and subexponential norm.  For a definition of a subexponential random variables, the subexponential norm and the subgaussian norm, we refer to \citet[Section 2.7]{Vershynin2018}. Note that all the constants in \citet[Proposition 2.7.1]{Vershynin2018}
		only differ by a constant. 
		\begin{equation*}
			\mathbb{E}\Big[\big( Y_i^2 - \mathbb{E}[Y_i^2]\big)^k \Big] = \mathbb{E}\Big[\big( (Y_i - \mathbb{E}[Y_i])^2 - \mathbb{E}[(Y_i - \mathbb{E}[Y_i])^2] + 2 \mathbb{E}[Y_i](Y_i - \mathbb{E}[Y_i]) \big)^k \Big]
		\end{equation*}
		Note that $\mathbb{E}[Y_i] \lesssim \operatorname{polylog}(n)$ as $x \in [-\log(n), \log(n)]^d$. Hence $Y_i^2 - \mathbb{E}[Y_i^2]$ can be represented by the sum of a subexponential and a subgaussian variable. Furthermore, if $Y_i - \mathbb{E}[Y_i]$ is subgaussian then it is also subexponential and the subexponential norm differs only by a constant from the subgaussian norm. Additionally, from \citet[Definition 2.7.5]{Vershynin2018} and the properties of a norm we learn that the sum of subexponential random variables is again subexponential and that we can bound the subexponential norm of the sum by the sum of the individual subexponential norms. Using that the subgaussian norm of $Y_i - \mathbb{E}[Y_i]$  is of order $\frac{\sigma_t}{\sqrt{\sigma_t^2+t^2} }$, we obtain using \citet[Proposition 2.5.2]{Vershynin2018}
		\begin{equation*}
			\mathbb{E}\Big[\big( Y_i^2 - \mathbb{E}[Y_i^2]\big)^k \Big]  \lesssim  \Big( \frac{\sigma_t^2}{t^2+\sigma_t^2}\Big)^{\frac{k}{2}} \leq \min\Big( \frac{\sigma_t^k}{t^k}, 1\Big).
		\end{equation*}
		
		Now we consider the case that $f$ is of the form \eqref{eq:derivative_t_5} or \eqref{eq:derivative_t_3}.
		First, we note that as shown in the proof of \Cref{thm:meh}, the distribution with density $q_{x,t}$ satisfies the Poincaré inequality with a constant
		$
		\rho \geq e^{-2L} \big(1+ \frac{t^2}{\sigma_t^2} \big).
		$
		An immediate consequence of the Poincaré inequality \Cref{Poincaré_inequality} is the following bound for all smooth functions $g, h$ 
		\begin{equation}\label{eq:derivate_t_9}
			\operatorname{Cov}(g(Y), h(Y))\leq \frac{1}{\rho} \sqrt{\mathbb{E}[|\nabla g(Y)|^2]\mathbb{E}[|\nabla h(Y)|^2]}.
		\end{equation}
		We are going to use this to bound
$
			\operatorname{Cov}(f(Y)( \frac{\partial}{\partial t} \log( w_{x,t}(Y)))^{k-1}, \frac{\partial}{\partial t} \log( w_{x,t}(Y))).
$
		Define
		$
		g(y) \coloneqq f(y)( \frac{\partial}{\partial t} \log( w_{x,t}(y)))^{k-1}.
		$
		Then 
		\begin{equation*}
			\nabla g(y) = (\nabla f(y))\Big( \frac{\partial}{\partial t} \log( w_{x,t}(y))\Big)^{k-1}+ f(y)(k-1) \nabla \Big( \frac{\partial}{\partial t} \log( w_{x,t}(y))\Big)^{k-2} \Big(\nabla \frac{\partial}{\partial t} \log( w_{x,t}(y))\Big)
		\end{equation*}
		and
		\begin{align}
			& |\nabla g(y)|^2  \lesssim |\nabla f(y)|^2 \Big|\frac{\partial}{\partial t} \log( w_{x,t}(y)))\Big|^{2(k-1)}+ |f(y)|^{2}(k-1)^2\Big| \Big( \frac{\partial}{\partial t} \log( w_{x,t}(y))\Big)\Big|^{2(k-2)} \Big|\nabla \frac{\partial}{\partial t} \log( w_{x,t}(y))\Big|^2 \notag\\
			& \lesssim |\nabla f(y)|^2 \big|\sigma_t^{-4(k-1))} |x-ty|^{2(k-1)}|y|^{2(k-1)}\big| +\big| \sigma_t^{-4(k-1)} \frac{1}{2} |x-ty|^{4(k-1)} (2\log(\sigma_{\min}) )^{2(k-1)}\big|\label{eq:derivate_t_6}\\& \quad  + | f(y)|^2\big(\big|\sigma_t^{-4(k-2)} |x-ty|^{2(k-2)}|y|^{2(k-2)}\big| + \big|\sigma_t^{-4(k-2)} \frac{1}{2^{2(k-1)}} |x-ty|^{4(k-2)} (2  \log(\sigma_{\min}) )^{2(k-2)}\big|\big) \label{eq:derivate_t_7}\\ & \qquad \cdot  |\sigma_t^{-2}(x-2ty + t(x-ty)2 \log(\sigma_{\min}))|^2.\label{eq:derivate_t_8}
		\end{align}
		All possible functions $f$ are polynomials in $y$, hence the derivatives are also polynomials. As $v_t(\cdot |\cdot)$ is linear in $y_i$ and $\frac{\partial}{\partial t} \log( w_{x,t})$ is quadratic in the components, the derivatives are of that degree or lower. We denote this by $\operatorname{poly}_2 (x,y,t) $. In case $f$ is of the form \eqref{eq:derivative_t_3} $\nabla f$ and $f$ are of order $\sigma_t^{-2}$, else the dependency on $\sigma_{\min}$ or $\sigma_t$ is only logarithmically. For the first term in \eqref{eq:derivate_t_6}
		\begin{align*}
			\mathbb{E}&[ |\nabla f(Y)|^2 \sigma_t^{-4(k-1)} |x-tY|^{2(k-1)}|Y|^{2(k-1)}\big ] \\& = \int \operatorname{polylog}(\sigma_{\min})\operatorname{poly}_2 (x,y,t) \sigma_t^{-4(k-1))} |x-ty|^{2(k-1)}|y|^{2(k-1)} \frac{ \exp(- \frac{|x-ty|^2}{2\sigma_t^2} - \frac{|y|^2}{2}- a(y))}{\int  \exp(- \frac{|x-ty|^2}{2\sigma_t^2} - \frac{|y|^2}{2}- a(y))\; \mathrm{d}y}\; \mathrm{d}y\\
			& \leq \frac{ e^{2L}  \operatorname{polylog}(\sigma_{\min})}{\int  \exp(- \frac{|z|^2}{2} - \frac{|\frac{x-\sigma_t z}{t}|^2}{2})\; \mathrm{d}y} \int\operatorname{poly}_2 \Big(x,\frac{x-\sigma z}{t},t\Big) \sigma_t^{-4(k-1)} |\sigma_t z|^{2(k-1)}\Big|\frac{x-\sigma_t z}{t}\Big|^{2(k-1)}  \exp\Big(- \frac{|z|^2}{2} - \frac{|\frac{x-\sigma_t z}{t}|^2}{2}\Big)\; \mathrm{d}y\\
			& \lesssim \operatorname{polylog}(\sigma_{\min}) \sigma_t^{-2k-2} |x|^{2(k-1)},
		\end{align*}
		where the last inequality stems from similar arguments as in the proof of \Cref{thm:meh} and \Cref{thm:error_decomp} for higher order moments. The exact calculations are omitted for the sake of brevity. The same calculation for the other terms in \eqref{eq:derivate_t_6}, \eqref{eq:derivate_t_7}, \eqref{eq:derivate_t_8} yield
		\begin{equation*}
			\mathbb{E}[|\nabla g(y)|^2] \lesssim \operatorname{polylog}(\sigma_{\min}) \log(n)^{4(k-1)} \sigma_t^{-2k-4} ,
		\end{equation*}
		where we used that $x \in [-\log(n), \log(n)]^d$ to combine the worst case dependencies. 
		For $h$, we obtain in similar manner
		$
		\mathbb{E}[|\nabla h(y)|^2] \lesssim \sigma_t^{-4} \log(n)^{4} .
		$
		Inserting these bounds in \eqref{eq:derivate_t_9}, we obtain
		\begin{equation*}
			\operatorname{Cov}(f(Y)( \frac{\partial}{\partial t} \log( w_{x,t}(Y)))^{k-1}, \frac{\partial}{\partial t} \log( w_{x,t}(Y))) \lesssim \operatorname{polylog}(\sigma_{\min})\operatorname{polylog}(n) \frac{\sigma_t^2}{\sigma_t^2 + t^2} \sigma_t^{-k-4}.
		\end{equation*}
		Hence in all cases, the terms occurring in the $k$-th derivative are bounded by
		
		\begin{equation*}
			\frac{\partial^k}{\partial t^k}v^j_t(x) \lesssim \operatorname{polylog}(\sigma_{\min})\operatorname{polylog}(n) \sigma_{\min}^{-k-2}       . \qedhere
		\end{equation*}
	\end{proof}
	
	\subsubsection{Helper results}
	\begin{lemma}
		\label{thm:helper_lemma}
		Let $\sigma_t = \sigma_{\min}^t$ with $\sigma_{\min}\in (0,1)$. Then for $t \in (0,1)$
		$
		\frac{t}{t^2+ \sigma_t^2}\leq \max(\log(\sigma_{\min}^{-1}), e^2).
		$
	\end{lemma}
	\begin{proof}
		First, let $t \geq \frac{1}{\log(\sigma_{\min}^{-1})}$. Then $
		\frac{t}{t^2+ \sigma_t^2}\leq \frac{1}{t} \leq \log(\sigma_{\min}^{-1}).
		$
		If $t < \frac{1}{\log(\sigma_{\min}^{-1})}$, then $
		\sigma_{\min}^{2t} = \exp(2\log(\sigma_{\min})t) \geq \exp(-2)
		$ and thus
		$
		\frac{t}{t^2+ \sigma_t^2}\leq   \frac{t}{t^2+ \exp(-2)}\leq \exp(2).
		$
	\end{proof}
	
	\begin{lemma}
		\label{thm:p_besov} For $p^*$ of the form \eqref{eq:assumption_p*} and any Lipschitz $1$ function $f\colon \mathbb{R}^d \rightarrow \mathbb{R}$, 
		\begin{equation}
			\Big|\int \nabla f(z) p^*(x-z) \mathrm{d} z-\int \nabla f(z) p^*(y-z) \;\mathrm{d} z\Big|\lesssim |x-y|.
		\end{equation}
	\end{lemma}
	\begin{proof}
		Using Hölders inequality we obtain 
		\begin{align*}
			\Big|\int \nabla f(z) p^*(x-z) \mathrm{d} z-\int \nabla f(z) p^*(y-z) \;\mathrm{d} z\Big| & \leq \| \nabla f\|_{\infty}  \int  |p^*(x-z) - p^*(y-z) |\; \mathrm{d} z\\
			& \leq \sqrt{d}  \int  \Big|\int_0^1 \langle x-y, \nabla p^*(x+t(y-x)-z ) \rangle \; \mathrm{d}t\Big|\; \mathrm{d} z,
		\end{align*}
		where the last inequality follows from the Lipschitz bound on $f$ and a backwards-application of the multivariate chain rule. The $\nabla$-operator is used with respect to the $\mathbb{R}^d$ valued input of $p^*$. Using the Cauchy-Schwarz inequality and changing the order of integration, we conclude
		\begin{align*}
			\int  \Big|\int_0^1 \langle x-y, \nabla p^*(x+t(y-x)-z ) \rangle \Big|\; \mathrm{d} z & \leq   |x-y| \int_0^1 \int   |\nabla p^*(w )|  \; \mathrm{d}w\; \mathrm{d} t.
		\end{align*}
		We can bound the integral over the derivative via
		\begin{equation*}
			\int   |\nabla p^*(w )|  \; \mathrm{d}w = \int|-w-\nabla a(w)|\frac{\exp(- \frac{-|w|^2}{2} -  a(w))}{\int \exp(- \frac{-|w|^2}{2} - a(w)) \; \mathrm{d}w}  \; \mathrm{d}w  \leq \mathbb{E}_{W \sim p^*}[|W|] + \sqrt{d}L.
		\end{equation*}
		We can further bound
		\begin{equation*}
			\mathbb{E}_{W \sim p^*}[|W|]  \leq e^{2L}  \mathbb{E}_{W \sim \mathcal{N}(0, I_d)}[|X|] \leq  e^{2L} \sqrt{\mathbb{E}_{W \sim \mathcal{N}(0, I_d)}[|X|^2]}  = e^{2L} \sqrt{d}.
		\end{equation*}
		Thus
		\begin{equation*}
			\Big|\int \nabla f(z) p^*(x-z) \mathrm{d} z-\int \nabla f(z) p^*(y-z) \;\mathrm{d} z\Big| \leq d (e^{2L}+L)|x-y|.\qedhere
		\end{equation*}
	\end{proof}

\end{document}